\documentclass[twoside]{article}

\usepackage[accepted]{aistats2024}
\usepackage{amsmath,amsfonts,bm}

\def\eqref#1{equation~\ref{#1}}

\def\1{\bm{1}}

\DeclareMathAlphabet{\mathsfit}{\encodingdefault}{\sfdefault}{m}{sl}
\SetMathAlphabet{\mathsfit}{bold}{\encodingdefault}{\sfdefault}{bx}{n}

\newcommand{\E}{\mathbb{E}}

\newcommand{\Var}{\mathrm{Var}}

\usepackage{hyperref}
\usepackage{url}

\usepackage{amsmath,amsthm,amssymb}
\usepackage{dsfont}
\usepackage{wrapfig}
\usepackage[noabbrev,capitalize]{cleveref}

\usepackage{bbm}

\usepackage{theoremref}

\usepackage{booktabs}

\usepackage{bm}
\usepackage{paralist}
\usepackage{enumerate}
\usepackage[inline]{enumitem}

\usepackage{mathtools}

\usepackage{setspace}
\usepackage{amsthm,amsmath,bbm,amsfonts,amssymb, mathrsfs}
\usepackage{enumerate}
\usepackage{cancel}
\usepackage{soul}
\usepackage{subcaption} 
\usepackage{booktabs}
\usepackage{adjustbox} 
\usepackage{etoolbox}  
\usepackage{accents}
\usepackage{apptools}
\usepackage{float}
\usepackage{scalerel,stackengine}
\usepackage{mathtools}
\usepackage{cleveref}
\usepackage{caption}
\usepackage{comment}
\usepackage{tikz}
\usetikzlibrary{intersections,arrows, arrows.meta, decorations.markings, matrix, calc, bayesnet, automata, chains, backgrounds, positioning, fit, shapes}

\newcommand{\N}{\mathcal{N}}
\newcommand{\M}{\mathcal{M}}

\def \bz {\boldsymbol{z}}
\def \bZ {\boldsymbol{Z}}
\def \bA {\boldsymbol{A}}

\usepackage{graphicx}

\usepackage{ifthen} %

\usepackage{caption}
\usepackage{subcaption}
\usepackage{algorithm, algorithmicx}
\usepackage[noend]{algpseudocode}

\usepackage{setspace}
\let\Algorithm\algorithm
\renewcommand\algorithm[1][]{\Algorithm[#1]\setstretch{1.0}}

\renewcommand{\epsilon}{\varepsilon}
\newcommand{\bX}{\mathbf{X}} %

\DeclareMathAlphabet\mathbfcal{OMS}{cmsy}{b}{n} %

\newcommand{\indep}{\bot\!\!\!\!\bot}

\theoremstyle{plain}
\newtheorem{lemma}{Lemma}
\newtheorem{corollary}{Corollary}
\newtheorem{proposition}{Proposition}

\newtheorem*{remark}{Remark}

\newtheorem{thm}{Theorem}

\newtheorem{rem}{Remark}

\usepackage[round]{natbib}

\makeatletter
\newcommand\notsotiny{\@setfontsize\notsotiny\@vipt\@viipt}
\makeatother

\begin{document}

%

%

\twocolumn[


\aistatstitle{A/B testing under Interference with Partial Network Information}

\aistatsauthor{ Shiv Shankar \And Ritwik Sinha \And Yash Chandak \And Saayan Mitra \And Madalina Fiterau }

\aistatsaddress{ UMass \And  Adobe Research \And Stanford University \And Adobe \And UMass} ]
\begin{abstract}
 A/B tests are often required to be conducted on subjects that might have social connections.  
 For e.g., experiments on social media, or medical and social interventions to control the spread of an epidemic.  
    In such settings, the SUTVA assumption for randomized-controlled trials is violated due to network interference, or spill-over effects, as treatments to group A can potentially also affect the control group B.
    When the underlying social network is known exactly, prior works have demonstrated how to conduct A/B tests adequately to estimate the global average treatment effect (GATE).
    However, in practice, it is often impossible to obtain knowledge about the exact underlying network.
    In this paper, we present UNITE: a novel estimator that relax this assumption and can identify GATE while only relying on knowledge of the \textit{superset} of neighbors for any subject in the graph.
    Through theoretical analysis and extensive experiments, we show that the proposed approach performs better in comparison to standard estimators.
\end{abstract}
\section{INTRODUCTION}

A/B tests, a form of randomized control trials, are the gold-standard in evaluating the impact of interventions, whether it be a new policy \citep{papadogeorgou2020causal}  or a medical treatment \citep{antman1992comparison}, or experimentation in the digital world \citep{siroker2015b}. A/B tests allow estimation of the treatment effect by ensuring that treatment (group A) and control (group B) assignments are made independently of other variables, including potentially unknown ones. The outcomes from the two groups are compared to determine the effect of treatment on a desired metric, such as health, clicks, or revenue. 

However, in many scenarios, several assumptions required for the A/B testing protocol are violated. Particularly, in many large-scale experiments, there might exist an underlying social network through which the exposures to the treatment group might also affect the subjects in the control group. Consider the example of social network choosing a new algorithm to deploy for user content recommendation. Evaluating the effectiveness of a single algorithm can be challenging as any content (e.g., music, news, travel suggestions) recommended to a user in the treatment group might get shared with, and thus affect, another user in the control group if the two users have a social connection \citep{brennan2022cluster,PougetAbadieSaveskiSaintJacquesDuanXuGhoshAiroldi17},
\citep{wong2020computational,kusner2016private}. If the social network wants to evaluate the effectiveness of shifting to the newer algorithm, it wants to evaluate the GATE effect (where all units are provided the new treatment). However, if one ignores the effect of the these interactions, we might underestimate the effect of a new algorithm, as users getting recommendations from the older algorithm are still getting exposed to the newer algorithm indirectly. Other situations where this problem can manifest itself are:

\textbf{(A) Privacy protection: } To protect user's privacy, digital services may not be able to use cookies or other trackers for user identification  \citep{shankar22cookie}. This can be problematic as users often access a digital service through multiple devices.  For instance, if a device is assigned to be in the treatment group and another device in the control group, then the outcomes on those devices may not be independent if those devices are used by a common user.

\textbf{(B) Epidemic control: } Herd immunity is a phenomenon, whereby the virtue of a sufficiently large vaccinated/treated group the entire community gains immunity against a disease. Thus, the outcomes of the treated might be similar to the outcomes of the control, thereby leading to bias in the treatment effect estimate, if the community effect is not adequately accounted for \citep{randolph2020herd,fine1993herd}.


This phenomenon of treatment to one unit affecting outcomes for other units is known as interference \citep{hudgens_halloran08,lesage2009introduction}. While interference-aware methods for treatment effect estimation exist, they often require knowledge of the interference structure \citep{ogburn2017causal,leung2020treatment}. 
However, in practice, it can be challenging or impossible to obtain the exact network structure for post-experimentation analysis using these interference-aware methods as well. Thus the scenarios mentioned above introduce a new challenge, necessitating the estimation of treatment effects from an \textit{uncertain} or a \textit{partial view} of interference structure.

\paragraph{Contributions.}

We focus on the challenge of treatment effect estimation in the presence of interference. We consider the setting where there is an interaction graph among subjects, and interference emanates from `neighboring' subjects. Compared to prior work, our focus is on scenarios where the exact  neighborhood information is unavailable. 

Our key contribution is to develop a method that identifies the Global Average Treatment Effect (GATE) using access to only the knowledge of a \textit{superset of these interfering units}, which is relatively easier to obtain in practice (e.g., using geographic locality). Furthermore our estimator's variance matches the mini-max optimal lower bound on MSE \citep{UganderKarrerBackstromKleinberg13,YuAiroldiBorgsChayes22} suggesting the theoretical efficiency of our estimate. Our paper not only establishes the theoretical validity of our estimator but also substantiates its practical efficacy through extensive experiments on both simulated and real-world datasets.


\section{RELATED WORK}
%
%
Network interference is a well studied topic in causal inference literature, with a variety of methods proposed for the problem. These often use varied set of assumptions about the exposure function \citep{AronowSamii17, auerbach2021local, li2021causal, viviano2020experimental} and interference neighborhoods \citep{bargagli2020heterogeneous, pmlr-v115-bhattacharya20a, SussmanAiroldi17, UganderKarrerBackstromKleinberg13}. 
Similar to these proposals, our approach also relies on imposing some structure to the form of network interference. A common assumption is that the network effects is linear with respect to a known functional of the neighbour treatments\citep{BasseAiroldi15,cai2015social,chin2019regression,GuiXuBhasinHan15,parker2016optimal,ToulisKao13}. 

Another common assumption is an exposure represented as the (weighted) proportion of either the neighboring units that have received treatment \citep{EcklesKarrerUgander17,ToulisKao13}, or the count of neighboring units that have undergone treatment \citep{UganderKarrerBackstromKleinberg13}. 
Frequently, a normalized exposure assumption is also applied, ensuring that for all units the net exposure  is in $[-1,1]$, where the extremes of $-1/1$ represent the unit being in control/ treatment group respectively \citep{leung2016treatment}. Such settings within the context of bipartite interference have been explored by researchers such as \citet{PougetAbadieSaveskiSaintJacquesDuanXuGhoshAiroldi17}, \citet{pouget2019variance}, and \citet{brennan2022cluster}. 
Furthermore, when the exposure model is known, various cluster randomized designs have been proposed for variance reduction in the estimate \citep{EcklesKarrerUgander17, GuiXuBhasinHan15,pouget2019variance}.
The key limitation of these methods is that one needs both the exact network structure and exposure model to compute the correct neighborhood statistics. 



Recently, some methods have been proposed based on multiple measurements which can address interference without knowing the network \citep{shankar22cookie,YuCortezEichhorn22,YuAiroldiBorgsChayes22}. These methods use a multiple-trial approach for estimating treatment effect. Assuming stationarity i.e. that the outcomes do not vary between the trials, multiple trials simplify GTE estimation by providing access to both the factual and counterfactual outcome. However, under non-stationarity (e.g. seasonal effects) their estimator is not valid. Moreover, such a model is unrealistic for our motivating use case of continuous optimization. Finally, in the more general setting, conducting multiple trials itself is fundamentally impossible \citep{shankar23diet}. As such, we want to develop a method which can work with observational data from a single RCT.

Other works  on treatment effect estimation under uncertain or mis-measured data include those focused on sensitivity analysis \citep{liu2013introduction,richardson2014nonparametric,veitch2020sense,dorie2016flexible} or on trying to identify confounders~\citep{ranganath2018multiple,d2019multi,wang2019blessings,miao2020identifying}. \citet{yadlowsky2018bounds} propose a loss-minimization approach that quantifies bounds on the conditional average treatment effect (CATE). They focus on analysis of errors in outcomes, confounder models, etc. but not on error in the interference graph.

\begin{figure*}
    \centering
    \includegraphics[width=0.9\textwidth]{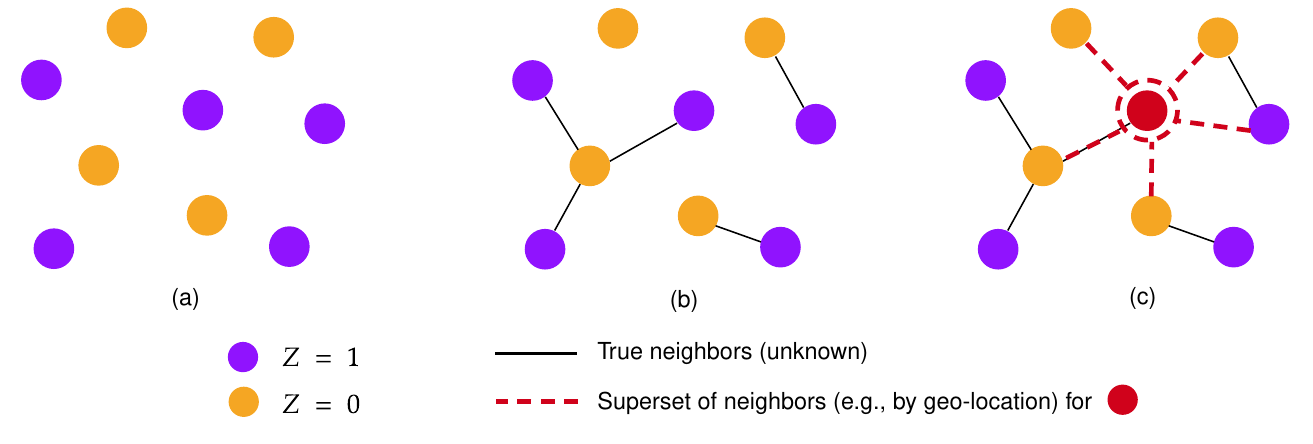}
    \caption{$Z=1$ denotes to units in the treatment group and $Z=0$ denotes units in the control group. \textbf{(a)} Standard A/B testing where there is no interaction between the treatment and the control units.\textbf{(b)} Network interference due to (unknown) interaction between the units. \textbf{(c)} We do not assume access to the exact underlying interaction graph. Instead, we consider a practically feasible assumption where only a \textit{superset} of neighbors for any node is available.  
\label{fig:1}} 
\end{figure*}

\section{UNITE}
In this section, we first formalize the treatment effect estimation problem with uncertainty about the interference graph. We then  present the assumptions that underlie our method. Finally, we describe our estimation method called \textit{UNITE} (\underline{U}ncertain \underline{N}etwork \underline{I}nterference Aware \underline{T}reatment-effect \underline{E}stimator), and present theoretical results about the method.

\subsection{Notation and Assumptions}
Let there be $n$ units in the population, $\bZ$ be the treatment assignment vector of the entire population, $\mathcal{Z}$ denote the set of possible treatment assignments, e.g., for binary treatments $\mathcal{Z} = \{0,1\}^n$. 
We use the Neyman potential outcome framework \citep{Neyman1923,rubin1974estimating}, and let $Y_i(\bz)$ denote the potential outcome of the $i$-th for $\bz \in \mathcal{Z}$.

Observations are made only at the unit level and are denoted as $Y_i$ for unit $i$. 
Note that multiple units might have a common factor influencing them. For example, if the units are users, their choice can be influenced not only by their experience on a website, but also by the experience of their social circle. Similarly, for a user using multiple devices, their behaviour at one device can be influenced by the treatment at the other device. As such the potential outcome at unit $i$ need not depend only on its own treatment assignment but also on other treatments allocated to its neihgbours.
This is a violation of the SUTVA assumption \citep{cox1958planning,hudgens_halloran08}; and is commonly called interference or spillover.

The dependence between unit level outcomes, can be encoded with its adjacency matrix $\bA \in \mathbb{R}^{n\times n}$, with $A_{ij} = 1$ only if treatment at unit $j$ can influence outcomes at unit $i$ ( e.g. if units $j$ is a friend of unit $i$ or if $i$ and $j$ are used by the same user).
Let $\mathcal{N}_i=\{j: A_{ij}=1\}$ be the set of \textit{neighbors} of device $i$, where by convention we include the self node i.e. $A_{ii} = 1$. We assume that the outcomes depend on the total treatments received by a unit through the graph. This implies that the interference at a device is limited to its immediate neighbours in the graph (i.e. SUTVA holds at the network level).
\begin{remark}
 While we consider only immediate neighbours in our descriptions, this is purely for descriptive convenience. One could simply extend our estimator to include $k$-hop neighbours instead of immediate neighbours.
 \end{remark}
\begin{align*}
\label{eq:neighbor_interference}
\text{If} \; \; z_i=z'_i \; \;  \text{and} \; \; z_j = z'_j \; \; \forall j \in \mathcal{N}_i, \; \; \text{then} \; \;
         Y_i(\bz) = Y_i(\bz').
 \end{align*}

The desired causal effect (GATE) is the mean difference between the outcomes when $\bz=\vec{1}\, (\text{i.e.,} \,  \forall i,\, z_i=1)$ and when $\bz=\vec{0} \, (\text{i.e.,} \, \forall i,\,z_i=0)$. Under the aforementioned notations, GATE is given by: 
\begin{equation}    \label{eq:definition_causal_effect}
    \tau(\vec{1}, \vec{0}) = \frac{1}{n}\sum_{i=1}^{n}Y_i(\vec{1}) - \frac{1}{n}\sum_{i=1}^{n}Y_i(\vec{0}) 
\end{equation}
To estimate $\tau(\vec{1}, \vec{0})$, we consider the following assumptions, that can broadly categorized under the following three categories.

\textbf{(I) Partial Graph: } If the true graph is known, prior works have provided an estimate of $\tau(\vec{1}, \vec{0})$ ~\citep{hudgens_halloran08,halloran2016dependent}.
However, access to $\bA$ is not possible in many settings.
Therefore, we will assume that the \textbf{exact} graph $\bA$ is \textbf{not} known. Instead, we assume access to a model $\mathcal{M}$ which provides constraints on $\bA$. Specifically, we assume $\mathcal{M}$ can be queried for a node $i$ to obtain a \textit{superset} $\mathcal{M}_i \supseteq \mathcal N_i$ of the true neighborhood.

In our use case of experimenting with social graphs, a superset of neighbors can be obtained in various ways: if the user has given cookie permissions,  or from some existing user model for identity linking, or even from something as basic as geo-location and ip addresses. This provides a significant practical advantage over the prior methods that necessitate knowledge of the \textit{exact} neighborhood.


\textbf{(II) Linear Additive Structure: }Randomized experiments with interference (even with neighbourhood interference) can be difficult to analyze since the number of potential outcome functions
for unit $i$ may be {$2^{|\mathcal{N}_i|}$}. This \textit{exponential} growth in number of neighbours $|\mathcal{N}_i|$ is problematic in comparison to the SUTVA case where there are only two outcome functions, irrespective of $|\mathcal{N}_i|$.
Therefore, the literature on network interference often restricts the set of potential outcome functions. \citep{EcklesKarrerUgander17,ToulisKao13,SussmanAiroldi17}. We follow a similar approach and let the outcome model be

\vspace*{-0.5cm}
\begin{align*}
\label{eq:linear_outcome}
         Y_i(\bz) = c_{i} + \sum_{j \in \mathcal{N}_i} c_{i,j}\mathbb{I}[{z_j=1}] + \epsilon \tag{\textbf{A1}}
 \end{align*}
where $\epsilon$ is an independent zero mean noise.
Under this assumption, GATE is given by $   \tau(\vec{1}, \vec{0}) = \frac{1}{n}\sum_{ i}\sum_{j \in  \mathcal{N}_i} c_{ij}$. We consider this linear model  to explain the core idea of our estimator. In Section \ref{sec:higher} we consider the generalized setting where $Y_i$ can be a non-linear function.




\textbf{(III) }We also make the following assumptions that are standard for the treatment effect and causal interference \citep{pearl2009causality,hudgens_halloran08,spirtes2010introduction}: 
\begin{align}
&\text{Network Ignorability:}\;\;  Y(\bz) \indep \bZ  &\: \forall \bz \tag{\textbf{A2}} \\
&\text{Positivity:} \;\; P(\bz | \bX) > 0 &\: \forall \bz \tag{\textbf{A3}} \\
&\text{Consistency:} \;\; Y_i = Y_i(\bz) &\text{ if } \bZ = \bz  \tag{\textbf{A4}}
\end{align}
Since we are in a A/B testing scenario, there do not exist any confounders. Moreover positivity is easily ensured by choosing a good randomization scheme. Hence the assumptions \textbf{A2-4} are naturally satisfied. 

\subsection{Estimation}
 \label{subsec:m2}

One of the most popular estimators for causal inference is the Horvitz-Thompson (HT) estimator \citep{horvitz1952generalization}. 
If the graph is known and when all treatment decisions are independent Bernoulli variables with probability $p$, an estimate $\tau_{\text{HT}}$ of GATE using HT estimator can be obtained as the following:
\begin{align}
 \tau_{\text{HT}} &= \frac{1}{n} \sum_i Y_i \left(  \prod_{j \in \mathcal{N}_i} \frac{z_j}{p} - \prod_{j \in \mathcal{N}_i} \frac{(1-z_j)}{(1-p)} \right).
 \label{eq:tau_ht}
\end{align}
However, using $ \tau_{\text{HT}}$  is practically infeasible because of the \textit{exponentially} high variance.
To observe the cause of high variance, consider that the graph $\bA$ is a k-regular graph and  $p=0.5$.  In such a case, the probability that any given node $i$ is exposed to only treatment or control group nodes in the treatment or control group is $1/2^k$. This falls off rapidly with $k$, requiring large graphs for even modest $k$. For example, with $k=20$, we need a graph of the order of a million nodes for the HT estimate to even have a meaningful value.
This problem of the HT-estimate is further compounded when the graph $\bA$ is not exactly known.

This raises a natural question: \textit{can the structure of the problem be exploited to reduce variance?} 

In the following, we show how the structure in \ref{eq:linear_outcome} can be leveraged to simultaneously address both the challenges: a) high variance of  $\tau_{\text{HT}}$ and b) incomplete knowledge of the graph $\bA$. We will first present a basic estimator, which we would then improve via incorporating ideas from self-normalized estimators \citep{thomas2017importance} and doubly-robust estimation \citep{chernozhukov2018double}.

\subsection{Variance Reduced Estimation}

We first have a look at how assumption \textbf{A1} affects $\tau_{\text{HT}}$. Substituting  \ref{eq:linear_outcome} in  \ref{eq:tau_ht}, $\tau_{\text{HT}}$ can be expressed as
\begin{small}
\begin{align*}
\frac{1}{n} \sum_i \left [ c_{i} + \sum_{j \in \mathcal{N}_i} c_{i,j}\mathbb{I}[{z_j=1}] \right] \left(  \prod_{k \in \mathcal{N}_i} \frac{z_k}{p} - \prod_{k \in \mathcal{N}_i} \frac{(1-z_k)}{(1-p)} \right).
\end{align*}
\end{small}
Now observe that as \textit{allocation} at each unit is independent, for any functions $g$ and $h$: $\E[h(z_i) g(z_j)] = \E[h(z_i)]\E[g(z_j)]$. Furthermore, as $\E[z_k/p] = \E[(1-z_k)/(1-p)] = 1$, we can ignore all the ratio terms for $k \neq j$ (see Appendix \ref{apx:HTconnection} for a complete derivation).

Therefore, $\tau_{\text{HT}}$  can be simplified as
{\small
\begin{align*}
\E[\tau_{\text{HT}}] = \frac{1}{n} \sum_i \E \left[ [ c_{i} + \sum_{j \in \mathcal{N}_i} c_{i,j}\mathbb{I}[{z_j=1}] ] \left(  \frac{z_j}{p} - \frac{(1-z_j)}{(1-p)} \right) \right],
\end{align*}
}%
which is a linear combination of in the terms $z_j/p$ and $(1-z)/(1-p)$. However, while this avoids the root cause of high variance (the product of the ratios), this expression cannot be computed from only the graph and observed outcomes $Y_i$. 

To resolve this problem, we will rewrite the earlier expression in terms of $Y_i$.  Observe that since  $z_j \indep z_i \; \forall i\neq j$, we can add terms of the form $z_i \left(\frac{z_j}{p} - \frac{1-z_j}{1-p}\right) \text{ with  } i \neq j$ without changing the expected value. Adding in such terms to include every node in $\M_i$, we get {\notsotiny
\begin{align*}
\E[\tau_{\text{HT}}] = \frac{1}{n} \sum_i \E \left[ \left( c_{i} + \sum_{j \in \mathcal{N}_i} c_{i,j}\mathbb{I}[{z_j=1}] \right) \left( \sum_{ k \in \M_i} \frac{z_k}{p} - \frac{(1-z_k)}{(1-p)} \right) \right]  
\end{align*}
}%

which motivates the following estimator:
\begin{align}
\label{eq:te_full}
\hat{\tau}_{\text{Lin}} = \frac{1}{n} \sum_i Y_i  \sum_{j \in \mathcal{M}_i} \left(  \frac{z_j}{p} - \frac{(1-z_j)}{(1-p)}  \right).
\end{align}

\begin{thm}

Under \textbf{A1-4}, and assuming $\mathcal{M}_i \supseteq \mathcal{N}_i$, $\hat{\tau}_{\text{Lin}}$ is an unbiased and consistent estimate of $\tau(\vec{1},\vec{0})$, i.e.,
$    \E [ \hat{\tau}_{\text{Lin}} ] = \tau(\vec{1},\vec{0}), 
$ and $\hat{\tau}_{\text{Lin}} \overset{a.s.}{\longrightarrow}\tau(\vec{1},\vec{0})$.
\thlabel{thm:lin}
\end{thm}
%


\begin{rem}
Observe that both \ref{eq:tau_ht} and \ref{eq:te_full} are linear combinations of outcomes $Y_i$ but the weights are different. While ${\tau}_{\text{HT}}$ uses weights that involve \textit{products} over the \emph{true} neighbors  $\N_i$ of a node, $\hat{\tau}_{\text{Lin}}$ use weights that are \emph{summed} over a \emph{superset} of neighbors $\M_i$. This allows a drastic reduction in variance.
\end{rem}
\begin{rem}
$\hat{\tau}_{\text{Lin}}$ is unbiased under the structural assumption \textbf{A1}, but ${\tau}_{\text{HT}}$ is unbiased irrespective. 
Intuitively, variance is traded-off with the identifiability of $\tau$ using the structural form of interference. 
\end{rem}




\paragraph{Self-Normalization}
While $\hat \tau_{\text{Lin}}$ can drastically reduce variance compared to $\tau_{\text{HT}}$, it can still be subject to high variance when $p$ is close to $0$ in the inverse-propensity $(z_i/p)$ terms. This fact is well known in the importance sampling (IS) literature where the inverse propensity is also known as `importance ratios'. 
A common alternative there is to use the self-normalized (or weighted) IS that can further reduce the variance at the cost of bias \citep{tukey1956conditional,rubin1987calculation}. 
%

Note that $\hat \tau_{\text{Lin}}$ bears resemblance to the general IS estimators, where in \ref{eq:te_full} the numerator in the ratios is the probability of treatments with the desired distribution (i.e., assign all units 0 or 1), and the denominator is the probability of the treatment under the sampling distribution. Using this insight, we can leverage self-normalization to further reduce the variance of $\hat \tau_{\text{Lin}}$.

Let $\rho_i = z_i/p$ and $\rho'_i = \frac{1-z_i}{1-p}$, 
then $\hat{\tau}_{\text{Lin}}$ can be rewritten with $\rho$, and $\rho'$ as:
\begin{align*} \hat{\tau}_{\text{Lin}} 
&= \frac{1}{n} \sum_i Y_i \sum_{j \in \mathcal{M}_i} \left(  \rho_j - \rho'_j  \right).
\end{align*}
Similar to self-normalized estimators, we can replace the importance weights $\rho$ by their self-normalized values $\bar{\rho}_i$ and $\bar{\rho}'_i$, where
\vspace*{-0.3cm}
\begin{align}
\label{eq:self_normal_rho}
\bar{\rho}_i &= \dfrac{ \rho_i}{\frac{1}{n}\sum\limits_j \frac{1}{|M_j|}\sum\limits_{k\in M_j} \rho_k}, \!\!\!&
\bar{\rho}'_i = \dfrac{ \rho'_i}{\frac{1}{n}\sum\limits_j \frac{1}{|M_j|}\sum\limits_{k\in M_j} \rho'_k}.
\end{align}
%
%
The self-normalized estimator in our case is given by:
$$
\hat{\tau}^{1}_{WIS} =    
    \frac{1}{n}\sum_i Y_i
    \left( \sum_{k\in \mathcal{M}_i} \bar{\rho}_k - \bar{\rho}'_k\right).
$$
\begin{thm}
    Under \textbf{A1-4} and assuming $\mathcal{M}_i \supseteq \mathcal{N}_i$, 
    $\hat{\tau}^{1}_{WIS}$ is a biased but consistent estimate of $\tau(\vec{1},\vec{0})$, i.e.,
    $
     \E [ \hat{\tau}_{WIS}^1 ] \neq \tau(\vec{1},\vec{0})$, and $
    \hat{\tau}^{1}_{WIS} \overset{a.s.}{\longrightarrow}\tau(\vec{1},\vec{0}).$
\thlabel{thm:wis}
\end{thm}

\paragraph{Doubly Robust Estimate: } Through self-normalization, $\hat{\tau}^{1}_{WIS}$ reduces the variance at the cost of bias. To avoid incurring this bias, but still reduce the variance of $\hat \tau_{\text{Lin}}$, we now propose a doubly robust estimator that leverages partial estimates of $\tau(\vec{1},\vec{0})$. 


There are a plethora of methods that make stronger assumptions about the interference or graph structure (e.g., knowledge of the exact graph) \citep{brennan2022cluster,EcklesKarrerUgander17,leung2020treatment}. These methods will typically provide biased but lower variance estimates  when their assumptions do not hold.
However, we can use our method to create a `robustified' version of these existing estimators, which not only (a) side-steps the need to validate the assumptions such models make, but can also (b) potentially reduce the variance of our estimator.

Leveraging ideas from control-variate \citep{thomas2016data} and doubly robust learning \citep{chernozhukov2018double} literature,
%
given a (partial/incorrect) estimate of the potential outcome functions $f_0(X_i) \approx Y_i(\vec{0})$ and $f_1(X_i) \approx Y_i(\vec{1})$, we propose the following estimator:
\begin{align*}
\hat{\tau}_{DR} &= \frac{1}{n}\sum_i \Big[ f_1(X_i) - f_0(X_i) + \\
&\quad(Y_i - f_1(X_i))\sum_{\M_i}\rho_j - 
 (Y_i - f_0(X_i))\sum_{\M_i} \rho_j' \Big].
\end{align*}
Estimator $\hat{\tau}_{DR}$ is beneficial as it may reduce variance using (partial) models $f_0$ and $f_1$, and remain unbiased irrespective of how inaccurate $f_0$ and $f_1$ are.
\begin{thm}
    Under \textbf{A1-4} and assuming $\mathcal{M}_i \supseteq \mathcal{N}_i$,  $\hat{\tau}_{DR}$ is an unbiased and consistent estimate of $\tau(\vec{1},\vec{0})$,
    $
    \E [ \hat{\tau}_{DR} ] = \tau(\vec{1},\vec{0}),$ and $ \hat{\tau}_{DR} \overset{a.s.}{\longrightarrow}\tau(\vec{1},\vec{0}).$
\thlabel{thm:dr}
\end{thm}


\subsection{Non-linear Structural Model}
\label{sec:higher}


\paragraph{Motif Model}
In the earlier section, we considered a linear structural model from \ref{eq:linear_outcome}. However, in general, there may also exist interaction from subgraphs in the neighborhood of $i$.
Graphs for such interaction structure are also known as ``motif" models \citep{aittokallio2006graph,holland1974statistical}, as ``network motifs" are a common way to characterize all patterns of smaller network features among a set of nodes (see  Figure \ref{fig:net_motif} for examples). While this may seem restrictive, our assumption on the functional form of potential outcome is less stringent than that commonly used in literature \citep{sussman2017elements,brennan2022cluster}. Furthermore if the neighborhoods are small, the motif model can represent all possible functions \citep{yuan2021causal}, and hence can also capture model heterogeneity in interference.

\begin{figure}
    \centering
    \includegraphics[width=0.5\textwidth, trim={1cm 1.4cm 0.5cm 1.2cm},clip]{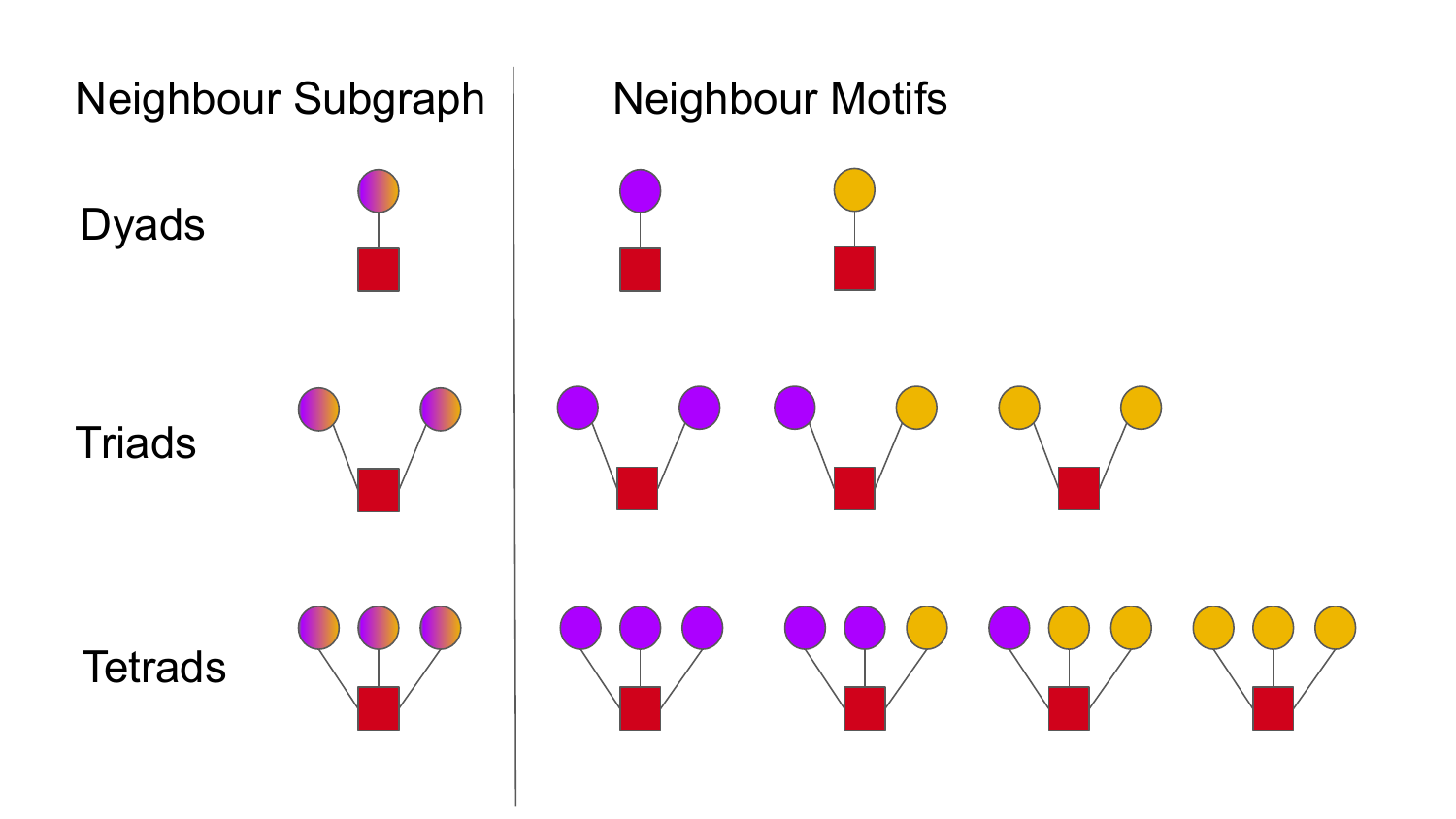}
    \caption{Examples of network motifs. Red square represent a node under consideration, and circles are its neighbours.  the left column lists possible subgraphs of the neighbourhood of a node. The right hand side depicts different motifs corresponding to these subgraphs. The yellow node represents $Z=0$ while purple represents $Z=1$. A motif is activated when its nodes are assigned the corresponding treatment assignment}
    \label{fig:net_motif}
\end{figure}

To model this setting, we consider outcomes to be a linear combination of \textit{influences} from small-sized ``network motifs", 
%
\begin{equation}
\label{eq:low_order_outcome}
         Y_i(\bz) =  c_{i}  + \sum_{ S \in \mathcal{SN}^\beta_i(\bA)} c_{i,S} \prod_{j \in S} \mathbb{I}[{z_j=1}] + \epsilon \tag{\textbf{A5}},
 \end{equation}
where $\mathcal{SN}_i(\bA)$ is a set of motifs of size up to $\beta$ for node $i$. It can equivalently be considered as a collection of subsets of $\N_i$ such that no set in $\mathcal{SN}_i(\bA)$ has size $> \beta$.


Note that, if we consider only dyads, i.e. we set $\beta=1$ and consider all such dyadic elements in $\N_i$ then $\textbf{A5}$ is equivalent to $\textbf{A1}$. Considering $\beta>1$ allows non-linear interaction between the nodes.




\paragraph{Estimator}
An important advantage of the estimators developed in the previous section is that they can be readily extended for the setting with non-linear interaction among the nodes. 
Specifically, let
\begin{align*}
\hat{\tau}^{\beta} &= \frac{1}{n} \sum_{i=1}^n Y_i \sum_{\substack{S \subseteq \M_i \\ |S| \leq \beta}} \Big( \prod_{j \in S} \frac{z_j-p}{p} - \prod_{j \in S} \frac{p - z_j}{1-p} \Big) 
\end{align*}

\begin{thm}
    Under \textbf{A2-5} and assuming $\mathcal{M}_i \supseteq \mathcal{N}_i$, $ \E [ \hat{\tau}^\beta ] = \tau(\vec{1},\vec{0})$, and $\hat{\tau}^\beta \overset{a.s.}{\longrightarrow}\tau(\vec{1},\vec{0}).$
\thlabel{thm:beta}
\end{thm}

\paragraph{Self-Normalized and Doubly Robust: }
Once again, expressing $\hat{\tau}^\beta$ using $\rho_i = z_j/p$ and $\rho'_i = \frac{1-z_j}{1-p}$,
\begin{align*}
\hat{\tau}^{\beta} 
&= \frac{1}{n} \sum_{i=1}^n Y_i \sum_{\substack{S \subseteq \M_i \\ |S| \leq \beta}} \Big( \prod_{j \in S} \big(\frac{z_j}{p} -1 \big) - \prod_{j \in S} \big(-1 + \frac{1 - z_j}{1-p} \big) \Big) \\
&= \frac{1}{n} \sum_{i=1}^n Y_i \sum_{\substack{S \subseteq \M_i \\ |S| \leq \beta}} \Big( \prod_{j \in S} \big(\rho_j -1 \big) - \prod_{j \in S} \big(-1 + \rho'_j \big) \Big).
\end{align*}
Let $\bar{\rho}_j$ and $\bar{\rho}'_j$ be the corresponding self normalized values of $\rho$ and $\rho'$, as in \ref{eq:self_normal_rho}. The self-normalized estimate for $\hat{\tau}^\beta$ is given by:
$$
\hat{\tau}^{\beta}_{\text{WIS}} = \frac{1}{n} \sum_{i=1}^n Y_i \sum_{\substack{S \subseteq \M_i \\ |S| \leq \beta}} \Big( \prod_{j \in S} \big(\bar{\rho}_j -1 \big) - \prod_{j \in S} \big(-1 + \bar{\rho}'_j \big) \Big)
$$

\begin{thm}
    Under \textbf{A2-5} and assuming $\mathcal{M}_i \supseteq \mathcal{N}_i$, 
    $ \E [\hat{\tau}^{\beta}_{WIS} ] \neq \tau(\vec{1},\vec{0})$, and $
    \hat{\tau}^{\beta}_{WIS} \overset{a.s.}{\longrightarrow}\tau(\vec{1},\vec{0})$
\thlabel{thm:betawis}
\end{thm}

Similarly, if we have access to a model for the counterfactual outcomes $f_0(X_i) \approx Y_i(\vec{0})$ and $f_1(X_i) \approx Y_i(\vec{1})$, we can create a DR estimator as
\begin{align}
\hat{\tau}^\beta_{DR} &= \frac{1}{n} \sum_{i=1}^n \tilde{Y}_i  \sum_{\substack{S \subseteq \M_i \\ |S| \leq \beta}} \Big( \prod_{j \in S} \big(\rho_j -1 \big) - \prod_{j \in S} \big(-1 + \rho'_j \big) \Big) \nonumber \\
& \quad \quad+ \frac{1}{n} \sum_{i=1}^n(f_1(X_i) - f_0(X_i)) 
\label{eq:beta_dr_def}
\end{align}
where $\tilde{Y}_i = Y_i -  z_i(f_1(X_i) - (1-z_i)f_0(X_i)$.
\begin{thm}
    Under \textbf{A2-5} and assuming $\mathcal{M}_i \supseteq \mathcal{N}_i$ , $\E[\hat{\tau}^\beta_{DR}] = \tau(\vec{1},\vec{0})$ and $\hat{\tau}_{DR}^\beta \overset{a.s.}{\longrightarrow}\tau(\vec{1},\vec{0})$. 
\thlabel{thm:betadr} 
\end{thm}


\paragraph{Application to General Non-Linearity}
Finally, the motif model in \ref{eq:low_order_outcome} paves the way for a general non-linear structural model.
In the following, we will use the index $j,k$ to refer to neighbors of node $i$. The dependence of their ranges and indices on $i$ is implicit and suppressed for notational ease.
Let the outcome for node $i$ be determined by the function $g_i$.
If $g_i$ is an analytic function, using Taylor-polynomials,
{\small
\begin{align}
Y_i(\bz) &=  g_i(z_1,z_2,..z_{|\mathcal{N}_i|}) \nonumber\\
  &= g_i(\vec{0}) +  \sum_{z_j}\left( \partial_{z^j} g_i\right) z_j  + 
  \sum_{z_j,z_k}\!\!\left(\partial^2_{z_j,z_k} g_i\right) z_jz_k  + O(\bz^3) \nonumber\\
  &\approx g_i(\vec{0}) + \!\! \sum_{z_j} \!\left(\partial_{z^j} g_i\right) z_j  + 
  \!\!\sum_{z_j,z_k} \!\!\!\left(\partial^2_{z_j,z_k} g_i\right) z_jz_k,  \label{eqn:nonlinear}
\end{align}
}%
%
which is of the same form as \textbf{A5} with $\beta=2$. Therefore, for interference under a general non-linear structural model, the estimate from $\hat \tau^\beta$ with $\beta=2$ may be biased but when $g_i$ can be assumed to be sufficiently smooth, its bias can be bounded.
%
\begin{thm} For the non-linear model \ref{eqn:nonlinear}, if $g_i$ is $k+1$ times differentiable and the $k+1$th derivative is bounded by $C$, then the absolute bias  $$\left|\E[\hat\tau^\beta] - \tau(\vec{1},\vec{0})\right| \leq \dfrac{C\max(p,1-p)}{(k+1)!}.$$
\thlabel{thm:bias}
\end{thm}
\begin{figure*}[th!]
    \centering
    \begin{subfigure}[b]{0.25\textwidth}
    \centering
    \includegraphics[width=\textwidth]{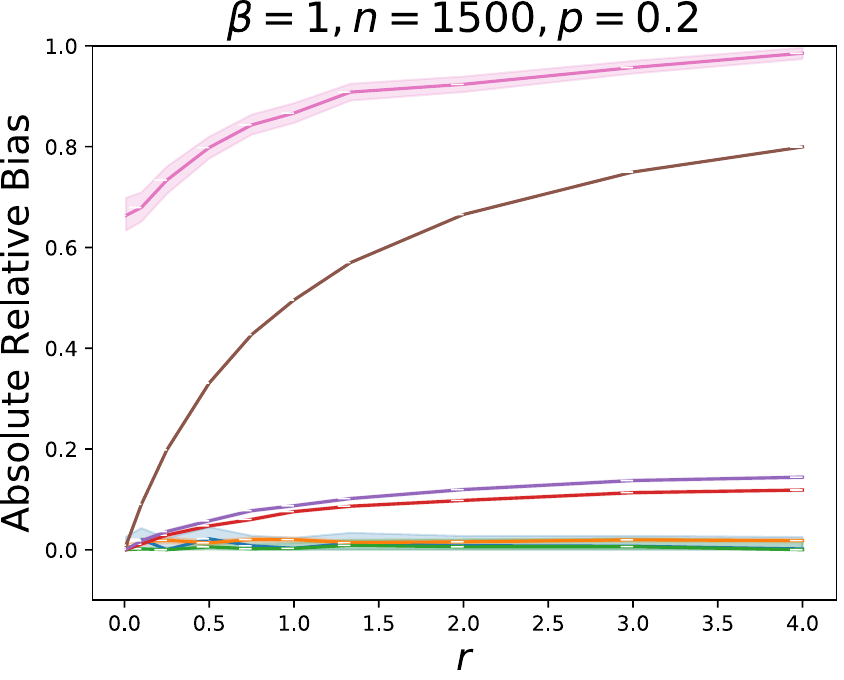}
    \includegraphics[width=\textwidth]{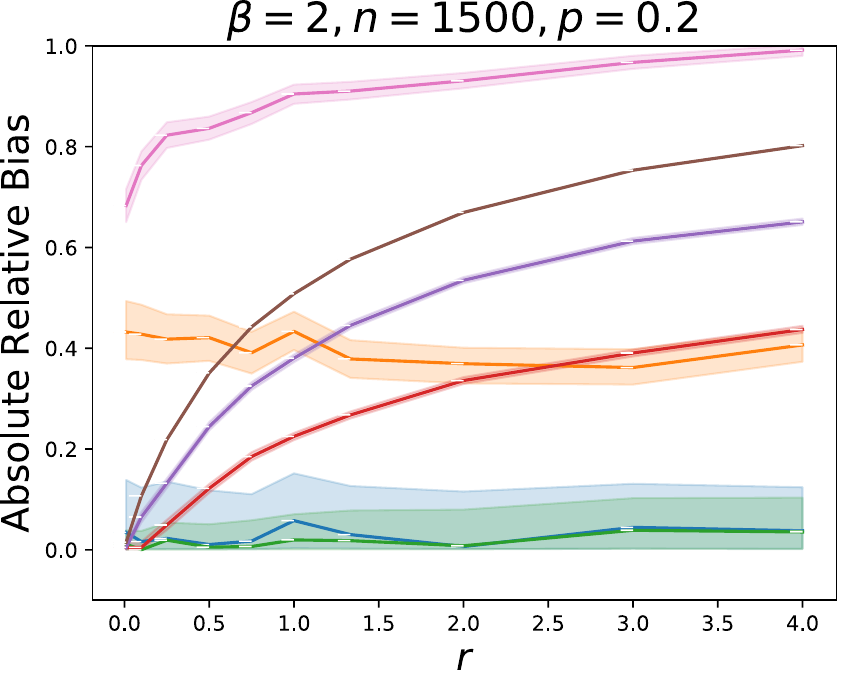}
    \includegraphics[width=\textwidth]{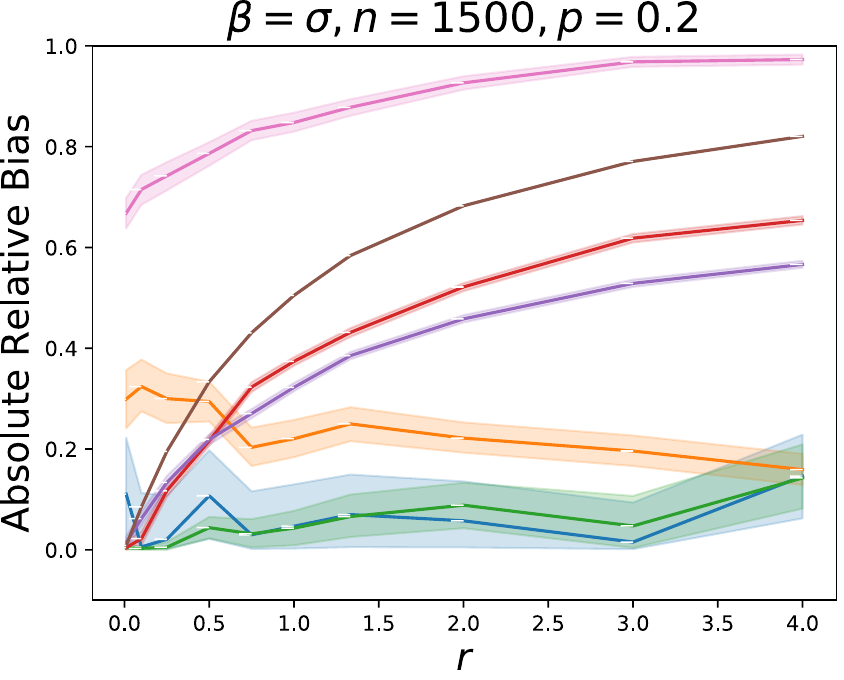}
    \caption{Direct/indirect effects: r }  \label{fig:ratioER}
    \end{subfigure}
    ~
    \begin{subfigure}[b]{0.25\textwidth}
    \centering
    \includegraphics[width=\textwidth]{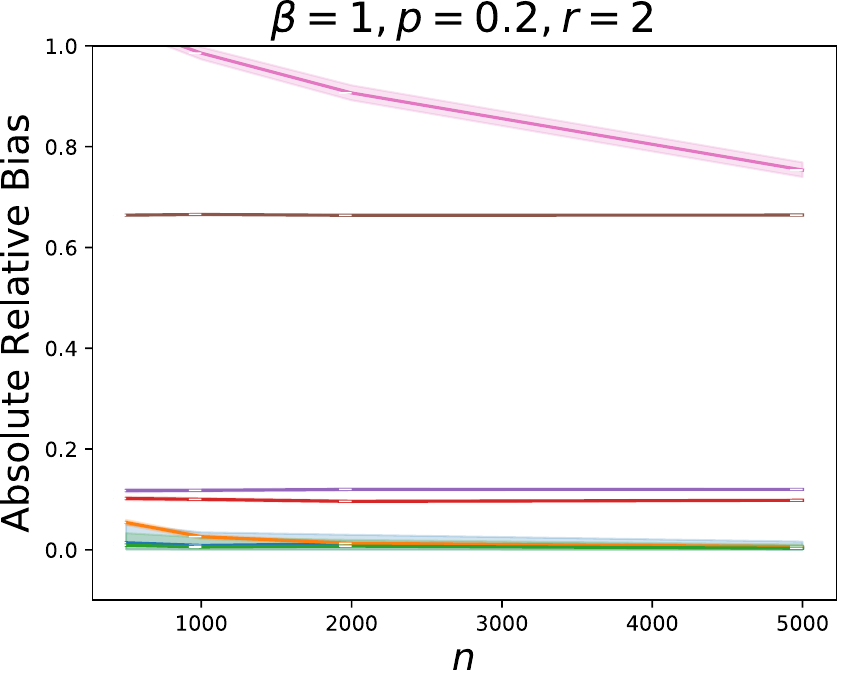}
    \includegraphics[width=\textwidth]{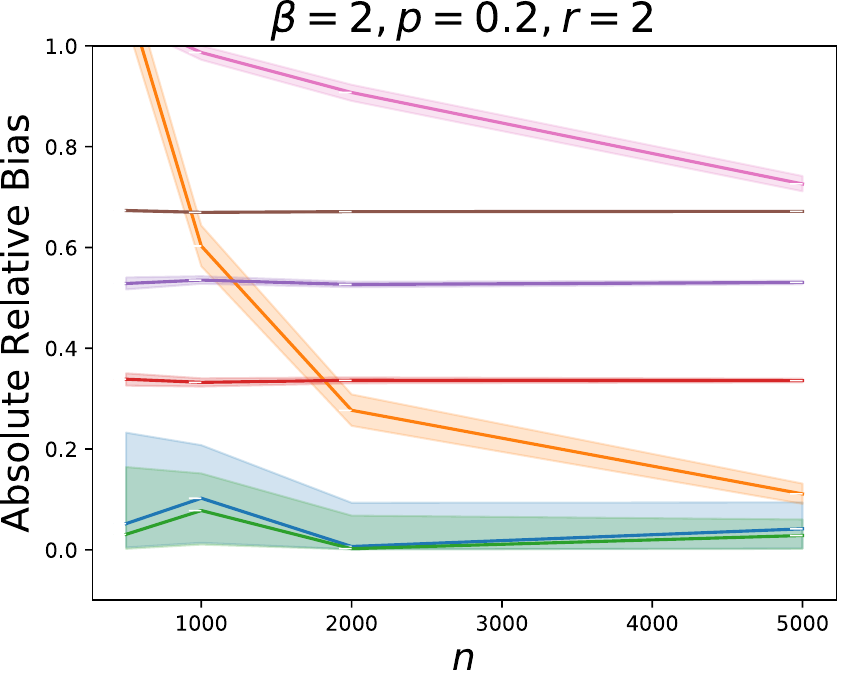}
    \includegraphics[width=\textwidth]{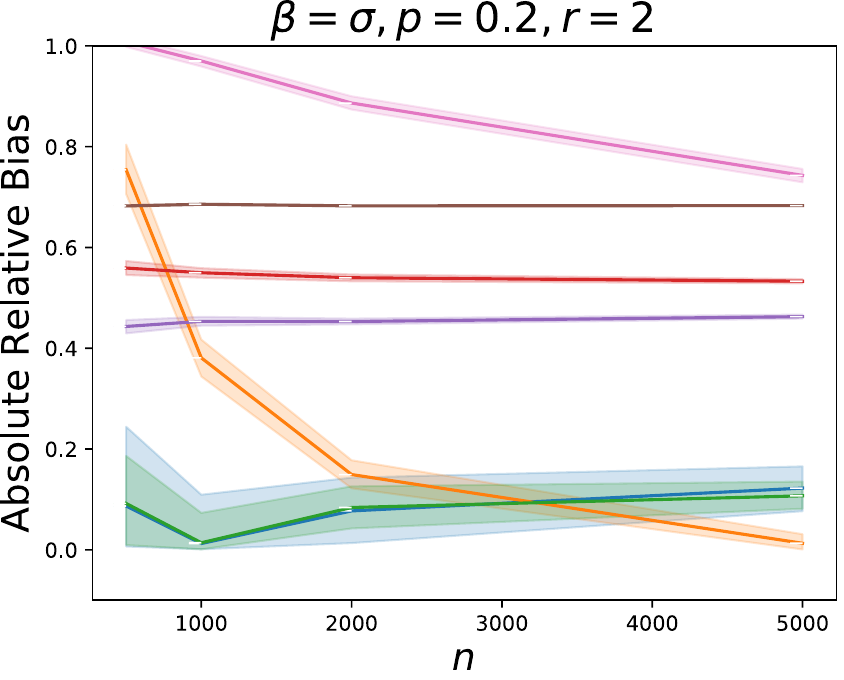}
    \caption{Population size: n}  \label{fig:sizeER}
    \end{subfigure}
    ~
    \begin{subfigure}[b]{0.25\textwidth}
    \centering
    \includegraphics[width=\textwidth]{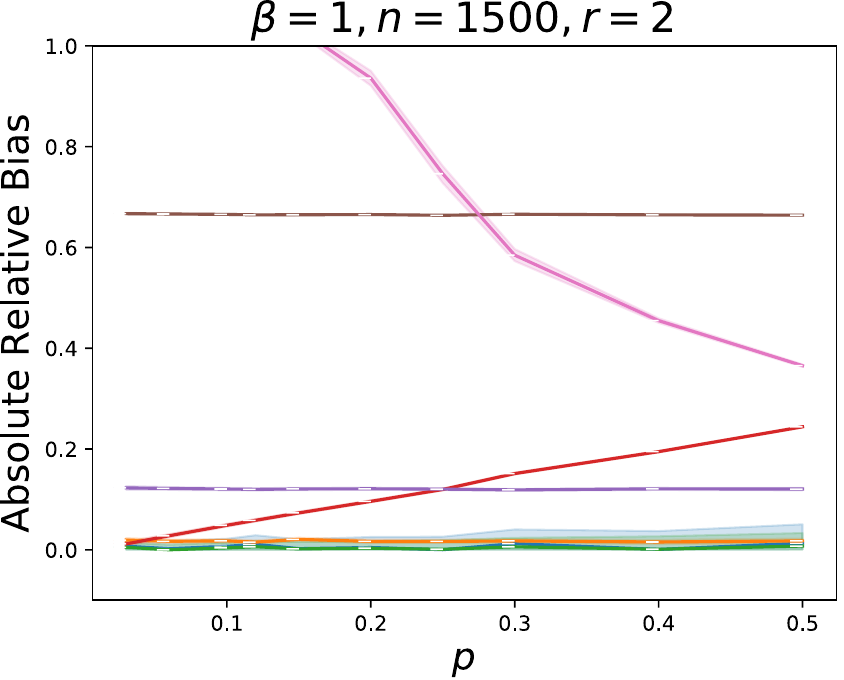}
    \includegraphics[width=\textwidth]{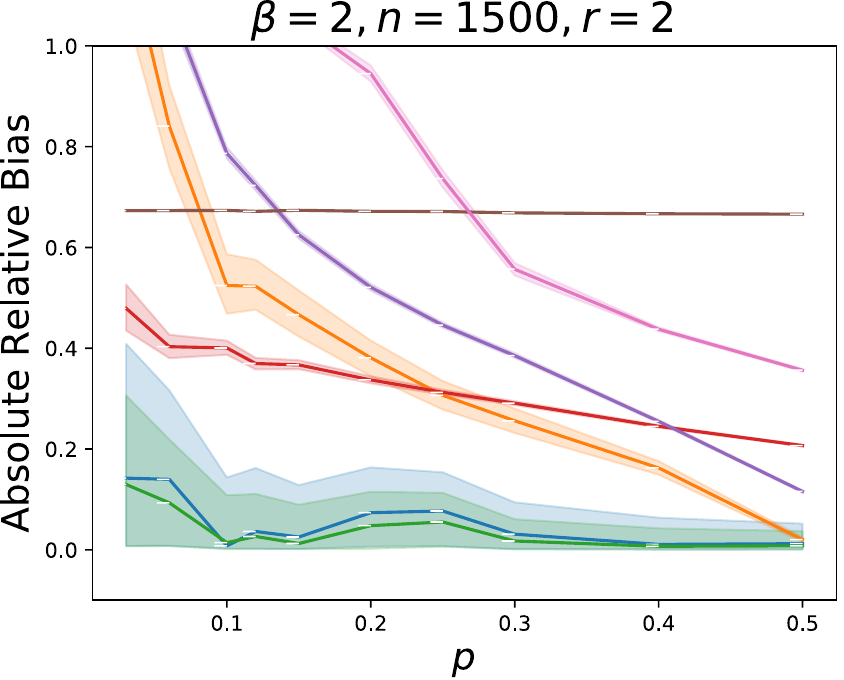}
    \includegraphics[width=\textwidth]{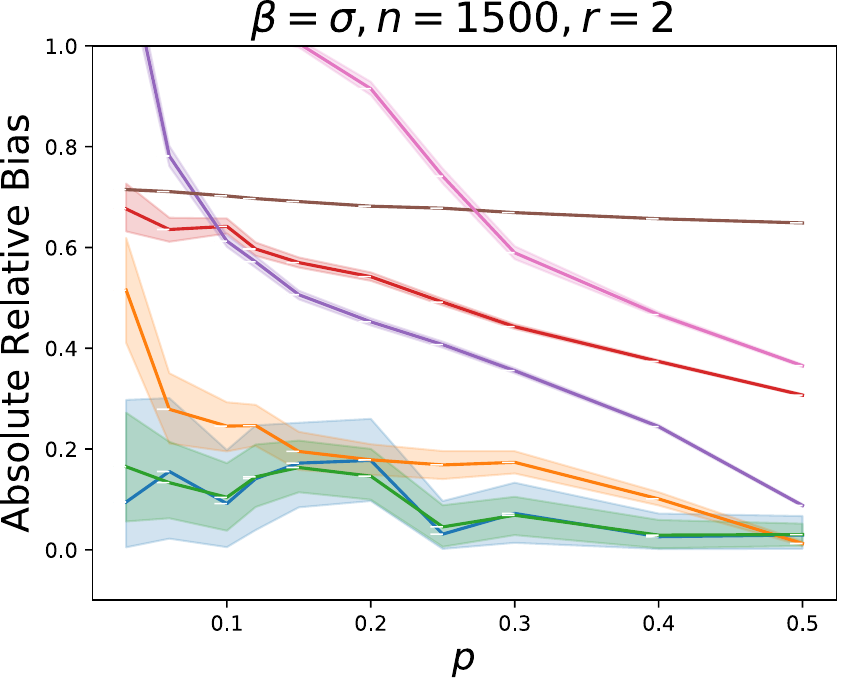}
     \caption{Treatment budget: p}  \label{fig:pER}
    \end{subfigure}
    \begin{subfigure}[b]{0.9\textwidth}
    \includegraphics[width=\textwidth]{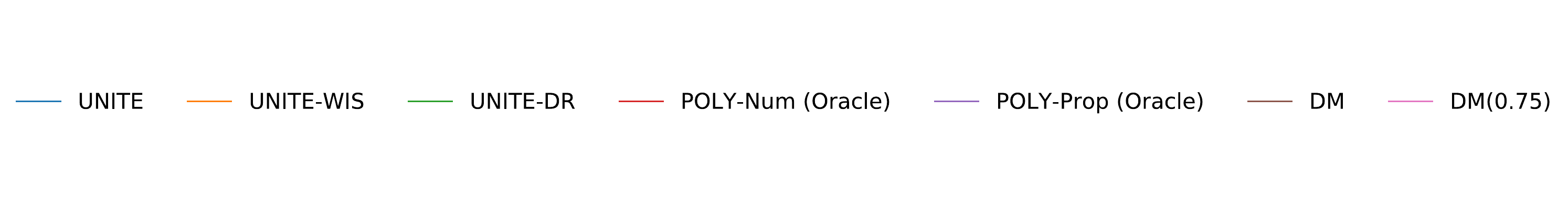}
    \end{subfigure}
    \vspace{-20pt}
    \caption{Performance of GATE estimators under Bernoulli design on
Erdos-Renyi networks. Y-axes represent the relative value of the absolute bias i.e. $\left| \frac{\E[\hat{\tau} - \tau(\vec{1},\vec{0})]}{\tau} \right|$, with the shaded width corresponding to the experimental deviation.  The rows correspond to linear, quadratic and sigmoid model for the potential outcomes.
Columns correspond to different parameters being varied: (a) Strength of interference,  where the x-axis corresponds to the average ratio of indirect to direct effects $r = \frac{1}{n}\sum_{i,j}\frac{|c_{ij}|}{{c_{ii}}|}$. (b) Population size,  where the x-axis corresponds to the number of nodes $n$ (c) Treatment budget, where the x-axis corresponds to the probability of treatment 1 ($p$).
\label{fig:erdos_combined}}
\end{figure*}

\subsection{Statistical Inference}
The results till now were focused with providing point-estimates of the treatment effect. However, in practice, one needs reasonable confidence intervals around these estimates, to handle statistical uncertainty and perform hypothesis tests to verify assumptions. For this purpose, we first provide bounds for variance of the unbiased estimator $\hat{\tau}^\beta$.

\begin{thm}
    Under \textbf{A2-5} and further assuming that $\mathcal{M}_i \supseteq \mathcal{N}_i$, if the $\epsilon$ noise have variance $\sigma$ then, $ \Var[\hat{\tau}^{\beta} ] \in O(1/n)$
\thlabel{thm:var_beta}
\end{thm}

Thus the UNITE estimator is consistent. We would like to highlight that the dependence is exponential in $d_\M^\beta$, where $d$ is related to the degrees of nodes in the graph. Incorporating higher order motifs can lead to higher variance, highly dense graphs have high variance and overall consistency requires a growth constraint on the graph density/degree. For the exact dependence on various graph parameters, we refer the reader to the Appendix. Similar bounds also hold for the DR and WIS estimators.

Next, we argue that this estimator is asymptotically normal. For this we rely on a classic result in generalized central limit theorems \citep{ross2011fundamentals}. Informally, for a set of $n$ bounded random variables $R_i$, if the dependency graph is not too dense, then the variance normalized sum approaches a normal distribution. The dependence between the variables is represented by the neighbourhoods in  $\M$. As such if $\M$ is not too dense, $\hat{\tau}^\beta$ is asymptotically normal.

Combined with the variance bound, the normality results suggests a way to do statistical inference. Specifically, since the variance formula is an upper bound, we can construct conservative Wald-type intervals \citep{wasserman2006all} without requiring the knowledge of the underlying graph. We should note however, that since the convergence is asymptotic, the use of the aforementioned variance for confidence intervals is only approximately valid. However, insights from \citet{sussman2017elements} affirm the minimax optimality of this bound concerning its dependence on $p$, $\sigma$, and the max-degree of $\M$ $\delta$. Consequently, these intervals maintain a robust reliability, yielding a level of precision considerably superior to that of the HT estimate \citep{AronowSamii17}.


\vspace{-5pt}
\section{EXPERIMENTS}

\subsection{Synthetic Interference Graphs}
In this section, we present synthetic experiments with interference on Erdos-Renyi graphs. The Erdos-Renyi (ER) model is commonly used for analyzing interaction networks in 
various experimental settings, particularly in the realm of social media \citep{seshadhri2012community} and epidemic control \citep{kephart1992directed,wang2003epidemic}. 
In social media platforms, where connections form organically, ER graphs provide a reasonable simulation of how friendships, followerships, or interactions might evolve in an online community \citep{erdhos1960evolution}. Additionally, in the context of epidemic control, ER graphs serve as valuable models for studying disease spread \citep{wang2003epidemic}. 

We simulate 50 different random Erdos-Renyi Graphs and run repeated experiments on these graphs with random treatment assignments. For these experiments, we provide an ablation study by varying the treatment probability, the strength of interference, and the size of the graphs to assess the efficacy of estimation across different ranges of parameters. We experiment with a linear outcome model $(\beta=1)$, a quadratic model $(\beta=1)$, and a non-polynomial model (labeled as $\beta=\sigma$).

\textbf{Baselines}
In our evaluation, we gauge the effectiveness of our proposed method by comparing it against commonly employed estimators such as polynomial regression (Poly), difference-in-means (DM) estimators. Since the polynomial regression model needs exact neighborhoods, we provide them with oracle access to the true interaction graph.\footnote{Due to incorporating large neighbourhoods, $\tau_{\text{HT}}$ failed to yield non-meaningful results in any trial.}

In Figure \ref{fig:erdos_combined}, we illustrate the relative bias of different estimators across a range of parameters.
Figure 3a plots the treatment effect estimate against the strength of interference,  \ref{fig:erdos_combined}b is for varying the network size $n$, and \ref{fig:erdos_combined}c plots it against treatment probability $p$.

As established in \thref{thm:lin}, UNITE models produce unbiased estimates for the linear (first row) and quadratic (second row) outcomes.  Note that in Figure \ref{fig:erdos_combined}a, when $r=0$, there is no interference, and hence most estimators are unbiased. However, when interference increases, baselines suffer from severe bias. 

As established in \thref{thm:wis}, the self-normalized version UNITE-WIS is biased, but its bias reduces as the number of nodes increases. It also shows a lower variance than UNITE as is common for self-normalized IS estimators. Finally, following \thref{thm:dr}, we see that UNITE-DR remains unbiased and shows some variance reduction over the vanilla UNITE estimate.

\textbf{Effect of General Non-Linearity}
In the third row of Figure \ref{fig:erdos_combined}, we illustrate results from a non-polynomial model, where the outcomes in this case come from the product of linear and a sigmoidal effect. It can be observed that UNITE shows bias, unlike in the linear and quadratic models where Assumption \textbf{A6} holds. However, UNITE still shows a lower bias than other estimators. Further, as it follows from \thref{thm:bias}, UNITE's bias reduces with increasing $p$.

\subsection{Case Study: Airbnb}

 We conduct experiments with a model designed from the AirBnB vacation rentals domain \cite{li2022interference}.  We perform simulations with protocol specified in \citet{brennan2022cluster}. Contrary to the previous experiments, the outcomes here do not follow an explicit exposure mapping. Instead, this simulation works uses a type matching model where if the listing and person have the same type, the probability of acceptance is higher. The measured outcome $Y_i$ is 1 iff the customer successfully books the place. The two different options are considered as recommendation algorithms, and treating each customer increases the application probabilities by a factor of $\alpha$. 
 

\begin{figure}[h!]
    \centering
    \centering
    \includegraphics[width=0.23\textwidth]{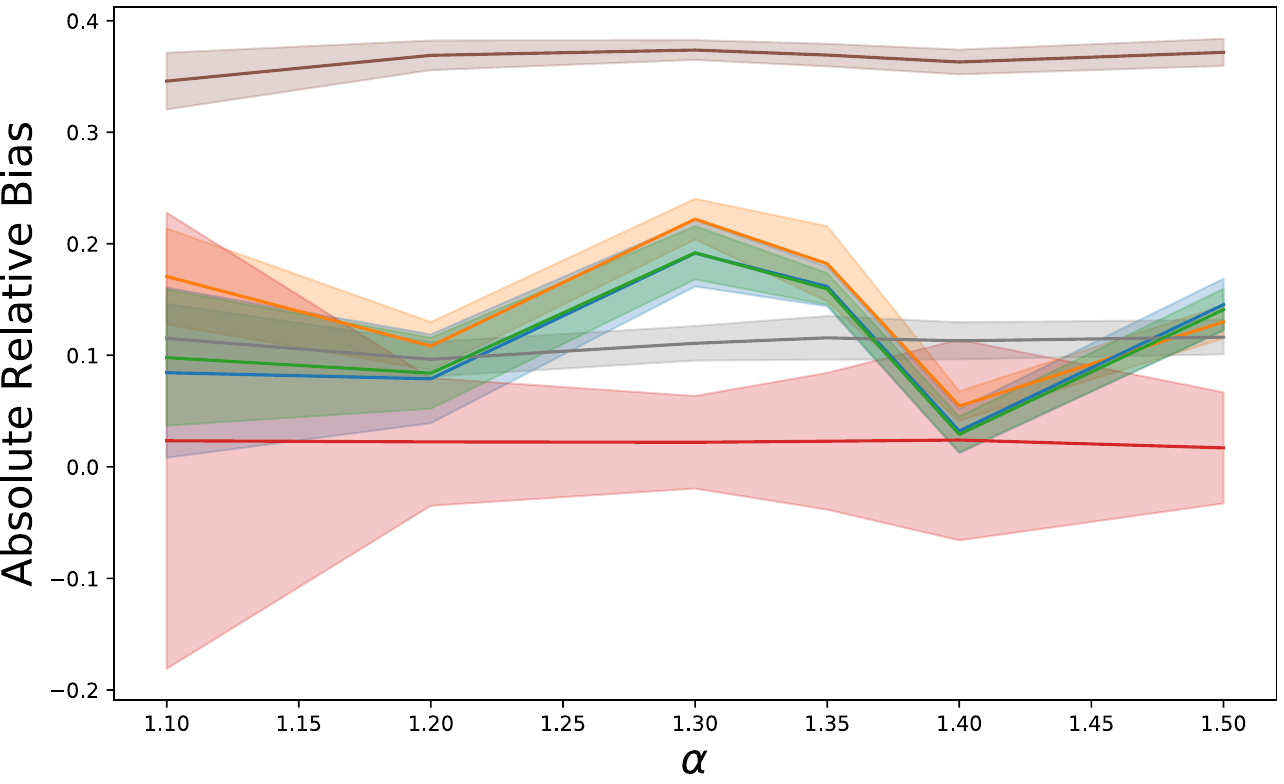}
    \includegraphics[width=0.23\textwidth]{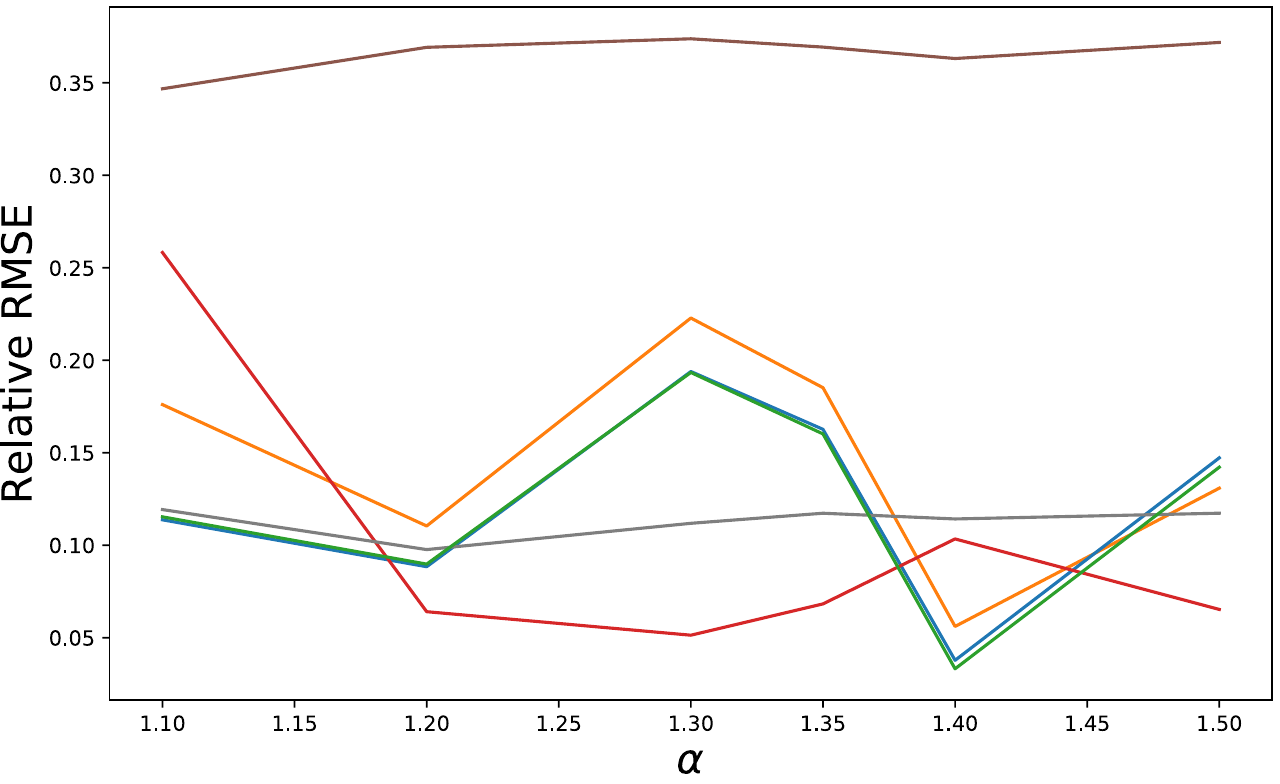}
    \includegraphics[width=0.45\textwidth]{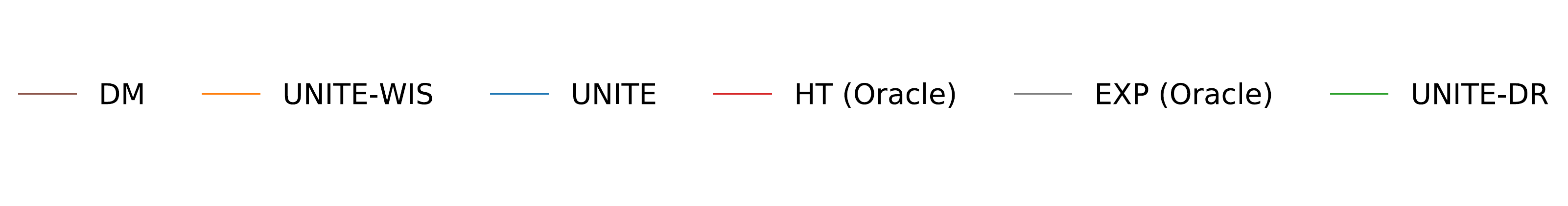}
    
    \caption{Visualization of relative bias and relative RMSE of different GATE estimators as the indirect treatment effect $\alpha$ increases.}
    
    \label{fig:airbnb}
\end{figure}

\paragraph{Baselines}
In this experiment, we evaluate our UNITE estimator against the difference-in-means (DM), as well as two oracle estimator which have access to the true graph. One is the oracle Horvitz-Thompson estimator. The other estimator is EXP (from \citet{brennan2022cluster}), which assumes both: a linear exposure model and access to the true graph. The EXP model is equivalent to ~\citep{AronowSamii17} estimator and computing it requires specifying both the true graph and an exposure function. The relative absolute bias of these estimators is seen in Figure \ref{fig:airbnb}.

Since the exposure model can only partly model the actual outcomes, there will be a necessary bias in this case. On the other hand, the Oracle HT estimator gives unbiased though higher variance estimates. From the result it is also clear that our approach works as well as the Oracle Exposure model. Our method has lower bias and similar RMSE compared to the EXP model, while requiring no such information.
\subsection{Effect of Network Uncertainty}

\begin{figure}[th!]
    \centering
    \begin{subfigure}[b]{0.23\textwidth}
    \centering
    \includegraphics[width=\textwidth]{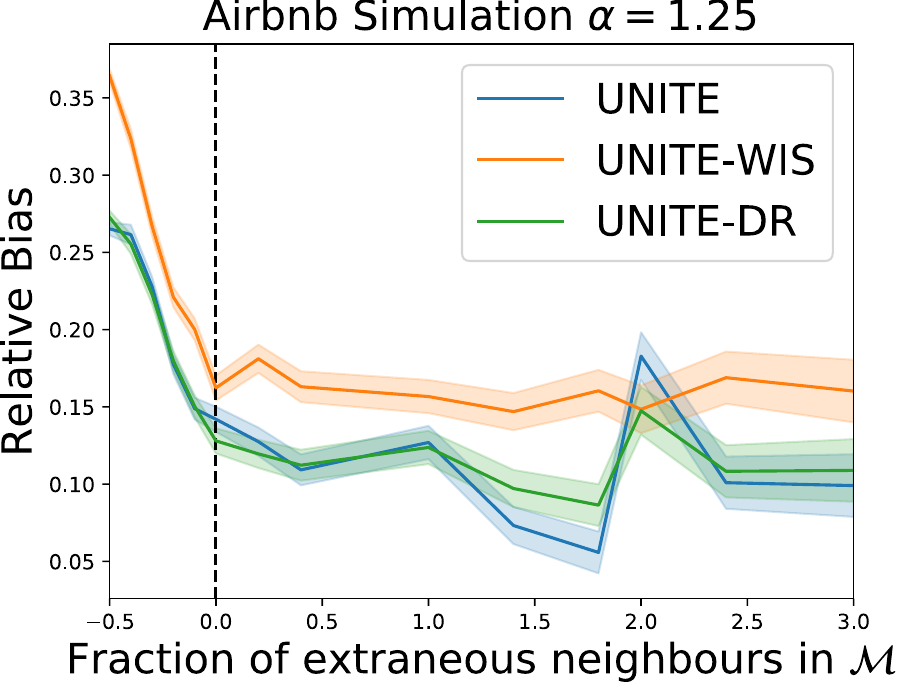}
    \caption{ Bias\label{fig:abl_ernet2}}
    \end{subfigure}
    ~
    \begin{subfigure}[b]{0.23\textwidth}
     \centering
     \includegraphics[width=\textwidth]{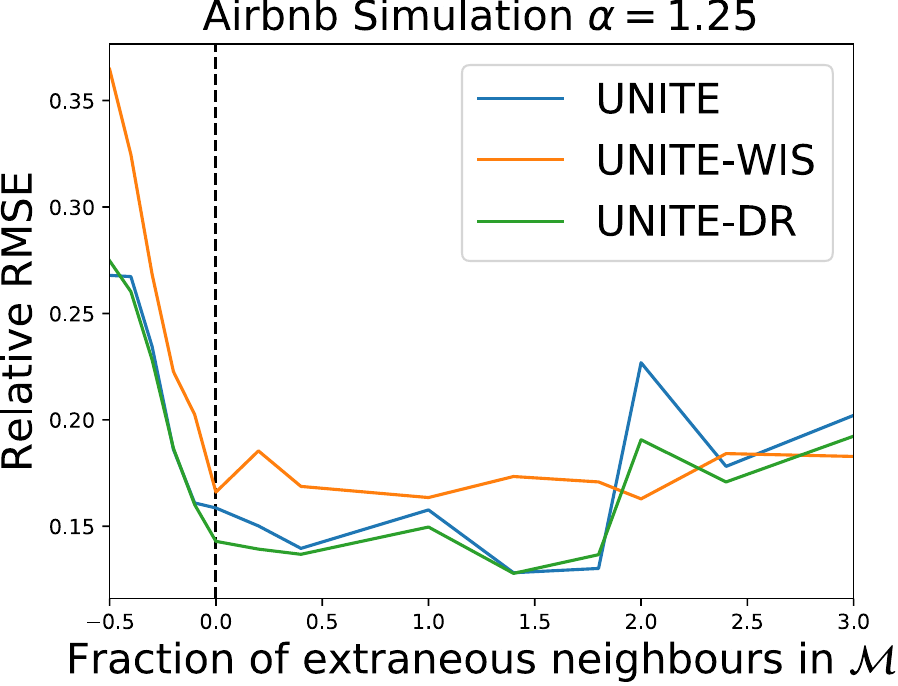}
     \caption{RMSE \label{fig:abl_airbnb}}
     \end{subfigure}
    \caption{Visualization of the impact of neighbourhood sizes on GATE estimation on the AirBnb Study. Negative fraction of neighbours indicate the case when $\mathcal{M}(i) \subset \mathcal{N}_i$ i.e. we missed pertinent neighbours. The bias tends to be high when given small neighborhoods, as they miss pertinent edges. As the neighborhood sizes increase, the bias reduces, but the uncertainty widens.
    \label{fig:ablation} }
\end{figure}

Figure \ref{fig:ablation} illustrates the impact of the neighborhood accuracy $\mathcal{M}(i)$ on estimating $\tau(\vec{0},\vec{1})$.
We experiment with Erdos-Renyi graphs as well as with the AirBnB Model. We fix the interference graph, and compute the treatment effect estimate from our method as we change the assumed neighbourhoods $\mathcal{M}(i)$. In Figure~\ref{fig:ablation}, we plot the absolute value of relative bias as varying proportions of edges are either added or omitted by $\mathcal{M}(i)$ for the AirBnB case. The results from ER graphs are in the Appendix. To maintain simplicity, we maintain uniform $\mathcal{M}(i)$ sizes across all nodes, employing the average number of missed or added edges as the metric along the x-axis.

We see that when $\mathcal{M}(i) \supseteq \mathcal{N}_i$ holds true for all nodes, UNITE is unbiased, and the variance of the estimate increases as the number of extraneous nodes within $\mathcal{M}(i)$ grows. But as expected, when $\mathcal{M}(i)$ misses relevant nodes, the estimate becomes biased, with the overall bias dependent on the influence induced by the missing nodes.

\section{CONCLUSION}
Network interference exists in many important A/B testing experiments. Our work provides estimators for GATE under a relaxed assumption of having knowledge only about the super-set of neighbors that cause interference. We believe that satisfying this relaxed assumption can be practically far more feasible than requiring the exact network. With both theoretical and experimental analysis, we established the efficacy of our estimator(s) under this assumption. 

\bibliography{mybib}  

\appendix

\setcounter{lemma}{0}
\setcounter{thm}{0}
\onecolumn

\newcommand{\cR}{\mathcal{R}}
\newcommand{\cS}{\mathcal{S}}
\newcommand{\cT}{\mathcal{T}}
\newcommand{\cU}{\mathcal{U}}
\newcommand{\cV}{\mathcal{V}}
\newcommand{\cW}{\mathcal{W}}
\newcommand{\Ind}{\mathbb{I}}


\section{Useful Lemmas}

\begin{lemma} \label{lem:exp_prod}
    Suppose that $\{z_i\}_{i=1..n}$ are mutually independent, with $z_i \sim \text{Bernoulli}(p)$. Then, for any set of indices $S, S' \subset[n]$, and function $f$ we have
    \[
        \E \Big[ \prod_{i \in S} \Big( \frac{z_i}{p} - \frac{1-z_i}{1-p} \Big) \prod_{j \in S'} f(z_{j}) \Big] =
        \begin{cases}
        (f(1)-f(0))^{|S\cap S'|} \E[f(z)]^{|S'\setminus S|} & \text{if } S \subseteq S'\\    
        0 & \text{otherwise}\\
        \end{cases}
    \]
\end{lemma}

\begin{proof}
Fix $S, S'$. A given index (node) $i$ can either be only in $S$ or only in $S'$ or in both, with only one of the possibilities being true. Correspondingly the product, $\prod_{i \in S} \Big(  \frac{z_i}{p} - \frac{1-z_i}{1-p} \Big) \prod_{j \in S'} f(z_{j}) $ can be factored into three exclusive products:

$$
 \prod_{i \in S} \Big( \frac{z_i}{p} - \frac{1-z_i}{1-p} \Big) \prod_{j \in S'} f(z_{j})  =  \prod_{i \in S \setminus S'} \Big( \frac{z_i}{p} - \frac{1-z_i}{1-p} \Big)  \prod_{k \in S \cap S'}  f(z_k) \Big(\frac{z_k}{p} - \frac{1-z_k}{1-p} \Big) \prod_{j \in S' \setminus S}  f(z_j)
$$
Applying expectations and noting that  $z_i$ are mutually independent, we get:
{\small
$$
\prod_{i \in S \setminus S'}  \E \big[ \frac{z_i}{p} - \frac{1-z_i}{1-p} \big]  \prod_{k \in S \cap S'} \E \big[ f(z_k )\Big(\frac{z_k}{p} - \frac{1-z_k}{1-p} \Big)\big] \prod_{j \in S' \setminus S}  \E f(z_j) = \prod_{i \in S \setminus S'} 0\prod_{k \in S \cap S'} \frac{\E[z_kf(z_k)] -p\E[f(z_k)]}{p(1-p)} \prod_{j \in S' \setminus S} \E[f(z_j)]
$$
}%
The RHS can only be non zero if $S \setminus S' = \{\}$ i.e. $S \subseteq S'$.

Since $\E\left[f(z_k)\left(\frac{z_k}{p} - \frac{1-z_k}{1-p} \right)\right] = p*f(1)*\frac{1}{p} + (1-p)*f(0)*(\frac{-1}{1-p}) = f(1) - f(0)$; thr RHS when it is non zero simplifies to
$$(f(1)-f(0))^{|S\cap S'|} \E[f(z)]^{|S'\setminus S|}$$
\end{proof}

\begin{corollary}
\label{lem:corr_id}
    By putting $f(z) =z$ in Lemma \ref{lem:exp_prod}we get
     \[
        \E \Big[ \prod_{i \in S} \Big( \frac{z_i}{p} - \frac{1-z_i}{1-p} \Big) \prod_{j \in S'} z_{j} \Big] =
        \begin{cases}
        p^{|S' \setminus S|} & \text{if } S \subseteq S'\\    
        0 & \text{otherwise}\\
        \end{cases}
    \]
\end{corollary}

\begin{lemma} \label{lem:help_beta}
    Suppose that $\{z_i\}_{i=1..n}$ are mutually independent, with $z_j \sim \text{Bernoulli}(p)$. Then, for any subsets $S,  S'$,
    $\E[\prod_{i \in S} f_i(z_i) \prod_{j \in \S'} \frac{z_j -p}{p}] = \prod_{i \in S \setminus S'} \E[f_i(z_i)]  \prod_{k \in S \cap S'}\left((1-p)(f_k(1) - f_k(0))\right) \Ind[S' \subseteq S] $
\end{lemma}

\begin{proof}
    Fix $S, S'$. A given index (node) $i$ can either be only in $S$ or only in $S'$ or in both, with only one of the possibilities being true. Correspondingly the product, $\prod_{i \in S} f_i(z_i) \prod_{j \in \S'} \frac{z_j -p}{p} $ can be factored into three exclusive products:

\begin{align*}
\E[\prod_{i \in S} f_i(z_i) \prod_{j \in \S'} \frac{z_j -p}{p}] &= \E[\prod_{i \in S \setminus S'} f_i(z_i)  \prod_{k \in S \cap S'} f_k(z_k)\frac{z_k -p}{p} \prod_{j \in S' \setminus S} \frac{z_j -p}{p}] \\
&= \prod_{i \in S \setminus S'} \E[f_i(z_i)]  \prod_{k \in S \cap S'}\E[f_k(z_k)\frac{z_k -p}{p}]\prod_{j \in S' \setminus S} \E[\frac{z_j -p}{p}] \\
&= \prod_{i \in S \setminus S'} \E[f_i(z_i)]  \prod_{k \in S \cap S'}\left((1-p)(f_k(1) - f_k(0))\right)\prod_{j \in S' \setminus S} 0 \\
&= \prod_{i \in S \setminus S'} \E[f_i(z_i)]  \prod_{k \in S \cap S'}\left((1-p)(f_k(1) - f_k(0))\right) \Ind[S' \subseteq S]
\end{align*}

Similarly,
{\small
\begin{align*}
\E[\prod_{i \in S} z_i \prod_{j in \S'} \frac{p - z_j}{1-p}] = \prod_{i \in S \setminus S'} \E[f(z_i)]  \prod_{k \in S \cap S'}\left((f(0) - f(1))p\right)\prod_{j \in S' \setminus S} 0 = \prod_{i \in S \setminus S'} \E[f(z_i)]  \prod_{k \in S \cap S'}\left((f(0) - f(1))p\right)
\Ind[S' \subseteq S]
\end{align*}
}%
\end{proof}

\begin{lemma} \label{lem:help_beta2}
For any sets $S',\M_i$ such that $S' \subseteq \N_i \subseteq \M_i$
$$
\E\left[ \prod_{ k \in S'} (z_k) \sum_{\substack{S \subseteq \M_i \\ |S| \leq \beta}} \Big( \prod_{j \in S} \frac{z_j-p}{p} - \prod_{j \in S} \frac{p - z_j}{1-p} \Big)\right]
= 1
$$
\end{lemma}
\begin{proof}

\begin{align*}
\E\left[ \prod_{ k \in S'} (z_k) \sum_{\substack{S \subseteq \M_i \\ |S| \leq \beta}} \Big( \prod_{j \in S} \frac{z_j-p}{p} - \prod_{j \in S} \frac{p - z_j}{1-p} \Big)\right]
&=  \sum_{\substack{S \subseteq \M_i \\ |S| \leq \beta}} \E\left[\left( \prod_{ k in S'} (z_k) \prod_{j \in S} \frac{z_j-p}{p} - \prod_{ k in S'} (z_k) \prod_{j \in S}   \frac{p - z_j}{1-p} \right)\right] \\
\intertext{Applying Lemma \ref{lem:help_beta} with $f_i(z) = z_i$ we get}
&= \sum_{\substack{S \subseteq \M_i \\ |S| \leq \beta}} \left[
p^{|S'/S]} (1-p)^{|S' \cap S|} \Ind [S \subseteq S'] - 
p^{|S'/S]} (-p)^{|S' \cap S|} \Ind [S \subseteq S']
\right] \\
&\overset{(b)}{=} \sum_{\substack{S \subseteq S' \\ |S| \leq \beta}} \left[
p^{|S'/S]} (1-p)^{|S' \cap S|}  - 
p^{|S'/S]} (-p)^{|S' \cap S|} 
\right] \tag{S1}\\
\intertext{\centering $(b)$ follows from that fact that $M_i \supseteq N_i$ for any node $i$ and $\Ind[S \subseteq S']$ will filter any non subset of $S'$}
&= \sum_{\substack{S \subseteq S' \\ |S| \leq \beta}} p^{|S'|} \left[
 p^{-|S|} (1-p)^{|S|}  -   p^{-|S|}(-p)^{| S|} 
\right]\\
&= \sum_{\substack{S \subseteq S' \\ |S| \leq \beta}} p^{|S'|} \left[ (\frac{1}{p} - 1)^{|S|}  -  (-1)^{| S|} 
\right] \tag{S2}\\
\intertext{Note that the terms in the sum S2 only depend on sizes of the subset $S$ and not the elements in it. Hence the sum S1 can be rewritten as:}
&= p^{|S'|} \sum_{\substack{r \leq \beta \\ r \leq |S'|}}  {|S'| \choose r}  \left[ (\frac{1}{p} - 1)^{r}  -  (-1)^{r}
\right] \\
&= p^{|S'|} \left[ \sum_{\substack{r \leq \beta \\ r \leq |S'|}}  {|S'| \choose r}  (\frac{1}{p} - 1)^{r}  - 
\sum_{\substack{r \leq \beta \\ r \leq |S'|}}  {|S'| \choose r}
(-1)^{r}
\right] \tag{S3}\\
\intertext{If $|S'| \leq \beta$, we are summing over all subsets of S', and the constraint of $r \leq \beta$ is redundant. Then by applying binomial theorem we get.}
&\overset{(c)}{=} p^{|S'|} \left[ \left( 1 + (\frac{1}{p} - 1) \right)^{|S'|} - \left(1 + (-1) \right)^{|S'|} \right] = p^{|S'|} (\frac{1}{p})^{|S'|} = 1
\end{align*}
where in $(c)$ we used binomial as $\sum_r {n \choose r}x^r = (1+x)^n$
\end{proof}

\section{Proof of general $\beta$ Estimator}

\begin{thm}
Under assumptions $\textbf{A2-5}$ and $\M_i \supseteq \N_i$, then $\hat{\tau}^\beta$ is unbiased 
\end{thm}

\begin{proof}
If $Y_i(\bz) = \sum_{S' \subset \N_i} c_{i,S'} \prod_{j \in S'} \Ind[z_j = 1]$ then for $\hat{\tau}^\beta$ we get

\begin{align*}
    \E[\hat{\tau}^\beta] &= \E \left[ \frac{1}{n} \sum_i Y_i \sum_{\substack{S \subseteq \M_i \\ |S| \leq \beta}} \Big( \prod_{j \in S} \frac{z_j-p}{p} - \prod_{j \in S} \frac{p - z_j}{1-p} \Big) \right] \\
    &= \E \left[ \frac{1}{n} \sum_i \sum_{S' \subset \N_i} c_{i,S'} \prod_{j \in S'} \Ind[z_j = 1] \sum_{\substack{S \subseteq \M_i \\ |S| \leq \beta}} \Big( \prod_{j \in S} \frac{z_j-p}{p} - \prod_{j \in S} \frac{p - z_j}{1-p} \Big) \right] \\
    &= \frac{1}{n} \sum_i \E \left[  \sum_{S' \subset \N_i} c_{i,S'} \prod_{j \in S'} \Ind[z_j = 1] \sum_{\substack{S \subseteq \M_i \\ |S| \leq \beta}} \Big( \prod_{j \in S} \frac{z_j-p}{p} - \prod_{j \in S} \frac{p - z_j}{1-p} \Big) \right] \\
    &= \frac{1}{n} \sum_i \E \left[   \sum_{S' \subset \N_i} c_{i,S'} \prod_{j \in S'} \Ind[z_j = 1] \sum_{\substack{S \subseteq \M_i \\ |S| \leq \beta}} \Big( \prod_{j \in S} \frac{z_j-p}{p} - \prod_{j \in S} \frac{p - z_j}{1-p} \Big) \right] \\
    &= \frac{1}{n} \sum_i \E \left[   \sum_{S' \subset \N_i} c_{i,S'} \prod_{j \in S'} z_j \sum_{\substack{S \subseteq \M_i \\ |S| \leq \beta}} \Big( \prod_{j \in S} \frac{z_j-p}{p} - \prod_{j \in S} \frac{p - z_j}{1-p} \Big) \right] \\
    &= \frac{1}{n} \sum_i  \sum_{S' \subset \N_i} c_{i,S'} \underbrace{\E \left[ \prod_{j \in S'} z_j \sum_{\substack{S \subseteq \M_i \\ |S| \leq \beta}} \Big( \prod_{j \in S} \frac{z_j-p}{p} - \prod_{j \in S} \frac{p - z_j}{1-p} \Big) \right]}_{\text{E1}}\\
    \intertext{Now applying Lemma \ref{lem:help_beta2} on E1 we get} 
    &= \frac{1}{n} \sum_i \sum_{S' \subset \N_i} c_{i,S'} \left[ 1 \right]  = \tau(\vec{1},\vec{0})
\end{align*}
\end{proof}

\begin{thm}
Estimator $\hat{\tau}_\text{Lin} $ is the same as $\hat{\tau}^\beta$ for $\beta=1$
\end{thm}

\begin{proof}
\begin{align*}
    \hat{\tau}^{\beta}|_{\beta=1} &= \frac{1}{n} \sum_i Y_i \sum_{\substack{S \subseteq \M_i \\ |S| \leq 1}} \Big( \prod_{j \in S} \frac{z_j-p}{p} - \prod_{j \in S} \frac{p - z_j}{1-p} \Big) \\
    & \frac{1}{n} \sum_i Y_i \sum_{\substack{S = \{k\} \\ k \in \M_i }} \Big( \prod_{j \in S} \frac{z_j-p}{p} - \prod_{j \in S} \frac{p - z_j}{1-p} \Big) \\
    &=\frac{1}{n} \sum_i Y_i \sum_{ j \in \M_i} \Big(  \frac{z_j-p}{p} -  \frac{p - z_j}{1-p} \Big) \\
    &=\frac{1}{n} \sum_i Y_i \sum_{ j \in \M_i} \Big(  \frac{z_j}{p} - \cancel{(1)} -  \cancel{(-1)} - \frac{1 - z_j}{1-p} \Big) \\
    &= \hat{\tau}_{\text{Lin}}
\end{align*}
\end{proof}


Together these theorems prove the unbiased part of Theorem \ref{thm:beta} and Theorem \ref{thm:lin}

\begin{thm}
Under assumptions $\textbf{A2-5}$, if $\M_i \cancel{\supseteq} \N_i$, then $\hat{\tau}^\beta$ can be biased.
    
\end{thm}

\begin{proof}

    We go straight to sum S1 in the proof of Lemma \ref{lem:help_beta2}. In doing the substituiton in the limits of the summation from $\M_i$ to $S'$ we used the fact that all subsets of $S'$ will be subsets of $\M_i$, and any non-subset of $S'$ is ignored due to $\Ind[S \subseteq S']$. If $\M_i \cancel{\supseteq} \N_i$, then $\exists i^* s.t. \:  i^* \in \N_i \text{ but } i^* \cancel{\in} \M_i$. Now take a $S' \subset \N_i$ such that $i* \in S'$.
    In the summation S1, now the limit of summation goes to subsets of $\M_i$
    \begin{align*}
   \E\left[ \prod_{ k \in S'} (z_k) \sum_{\substack{S \subseteq \M_i \\ |S| \leq \beta}} \Big( \prod_{j \in S} \frac{z_j-p}{p} - \prod_{j \in S} \frac{p - z_j}{1-p} \Big)\right]
&=  \sum_{\substack{S \subseteq S' \\ S \subseteq \M_i \\ |S| \leq \beta}} \left[
p^{|S'/S]} (1-p)^{|S' \cap S|}  - 
p^{|S'/S]} (-p)^{|S' \cap S|} \right] \\
&=  \sum_{\substack{S \subseteq S' \cap \M_i \\ |S| \leq \beta}} \left[
p^{|S'/S]} (1-p)^{|S' \cap S|}  - 
p^{|S'/S]} (-p)^{|S' \cap S|} \right] \\
\intertext{Following the exact same steps we get: }
&= p^{|S'|} \sum_{\substack{r \leq \beta \\ r \leq |S'|}}  {|S' \cap \M_i| \choose r}  \left[ (\frac{1}{p} - 1)^{r}  -  (-1)^{r} \right] \\
&= p^{|S'|} (\frac{1}{p})^{|S' \cap \M_i|} = p^{|S'\setminus \M_i|} \leq 1 \tag{\textbf{B1}}
\end{align*}

Now using definition of $\hat{\tau}^\beta$ we get
\begin{align*}
    \E[\hat{\tau}^\beta] &= \E \left[ \frac{1}{n} \sum_i Y_i \sum_{\substack{S \subseteq \M_i \\ |S| \leq \beta}} \Big( \prod_{j \in S} \frac{z_j-p}{p} - \prod_{j \in S} \frac{p - z_j}{1-p} \Big) \right] \\
    &= \E \left[ \frac{1}{n} \sum_i \sum_{S' \subset \N_i} c_{i,S'} \prod_{j \in S'} \Ind[z_j = 1] \sum_{\substack{S \subseteq \M_i \\ |S| \leq \beta}} \Big( \prod_{j \in S} \frac{z_j-p}{p} - \prod_{j \in S} \frac{p - z_j}{1-p} \Big) \right] \\
    &=  \frac{1}{n} \sum_i \sum_{S' \subset \N_i} c_{i,S'} \E \left[ \prod_{j \in S'} \Ind[z_j = 1] \sum_{\substack{S \subseteq \M_i \\ |S| \leq \beta}} \Big( \prod_{j \in S} \frac{z_j-p}{p} - \prod_{j \in S} \frac{p - z_j}{1-p} \Big) \right] \\
    \intertext{which by \textbf{B1} is}
    &=  \frac{1}{n} \sum_i \sum_{S' \subset \N_i} c_{i,S'} p^{|S'\setminus M_i|} \neq \tau(\vec{1},\vec{0})
\end{align*}

For any $S' \text{ with } |S'| < \beta $ which influences $Y_i$ but is not a subset of $\M_i$, the corresponding coefficient $c_{i,S'}$ is attenuated by a factor of $p^{|S'\setminus \M_i|}$ in the estimate $\hat{\tau}^\beta$

\end{proof}

\subsection{Variance Analysis}

Next, we provide an upper bound for the variance of this estimate. But before we do that we need another helpful Lemma.

\begin{lemma} \label{lem:help_beta_cov}
    Suppose that $\{z_i\}_{i=1..n}$ are mutually independent, with $z_j \sim \text{Bernoulli}(p)$. Then, for any subsets $T, S,  S'$, we have
    \begin{small}
    \begin{align*}
    \E[\prod_{i \in T} f_i(z_i) \prod_{j \in S} \frac{z_j -p}{p} \prod_{j \in S'} \frac{-z_j + p}{1-p}] &= \prod_{ T \setminus (S \cup S')} \E [f_i]
    \prod_{ (T \cap S) \setminus (S')} (1-p)(f_j(1) - f_j(0))
    \prod_{ (T \cap S') \setminus (S)}  p(f_j(0) - f_j(1)) \quad \times \\
    &\quad\prod_{ (T \cap S' \cap S)} -\E[f_i(1-z_i)]
    \prod_{ (S' \cap S) \setminus T} -1
    \Ind[S \Delta S' \subseteq T]
    \end{align*}
    \end{small}
\end{lemma}

\begin{proof}
    We first apply Lemma \ref{lem:help_beta} on the above expression, and see that it is non-zero only if $ S \subseteq (T \cup S')$.
    We can also apply the second result from Lemma \ref{lem:help_beta} and get that the expression is non-zero only if $ S' \subseteq (T \cup S)$.
    Since $S \cap S' \subseteq S,S'$, the two conditions imply that for the expression to be non-zero we need
    $S \Delta S' \subseteq T$

Next we need to identify the expected value of the combined function applied on a node as in Lemma \ref{lem:help_beta}. 

   \begin{itemize}
   \item For indices in $ T \setminus (S \cup S')$ we get $\E[ f_i]$ 
   \item For indices in $ S \setminus (T \cup S')$ we get 0. This is subsumed in  $S \Delta S' \subseteq T$
   \item For indices in $ S' \setminus (T \cup S)$ we get 0. This is subsumed in  $S \Delta S' \subseteq T$
   \item For indices in $ ( S \cap T) \setminus S'$ we get
   $\E[ f_i  \frac{z_j -p}{p}] = (1-p) (f_i(1) - f_i(0))$
   \item For indices in $ ( S' \cap T) \setminus S$ we get $\E[ f_i \frac{-z_j + p}{1-p}] = p(f_i(0) - f_i(1))$
   \item For indices in $ ( S' \cap S) \setminus T$ we get $\E[ \frac{z_j -p}{p} \frac{-z_j + p}{1-p}] = -1$
   \item For indices in $ T \cap S \cap S'$ we get
   $\E[ f_i \frac{z_j -p}{p} \frac{-z_j + p}{1-p}] =  -p f_i(0) - (1-p)f_i(1)= -\E[f_i(1-z_i)] $
   \end{itemize}
\end{proof}

\begin{thm}
Under assumptions $\textbf{A2-5}$ and assuming $\M_i \supseteq \N_i$, we have that
$\Var[\hat{\tau}^\beta] \leq \frac{1}{n} Y^2_{\max} (p(1-p))^{-\beta} d_\M^{1+\beta}  d_\N^{1+\beta}
$
\end{thm}
\begin{proof}

For simplicity of notation we denote by $\Psi(S)$ and $\Psi'(S)$ the products $\prod_{i\in S} \dfrac{z_i-p}{p}$ and $\prod_{i\in S} \dfrac{p-z_i}{1-p}$. We would also first ignore the additional noise $\epsilon$ while presenting the key derivation.

For any random variable $X$, $\Var[X] = \E[X^2] - \E[X]^2 \leq \E[X^2]$
Applying the same on $\hat{\tau}^\beta$ we get, 
\begin{align*}
\Var[\hat{\tau}^\beta] &\leq  \E \left[ \left( \frac{1}{n} \sum_{i=1}^n Y_i \sum_{\substack{S \subseteq \M_i \\ |S| \leq \beta}} \Big( \prod_{j \in S} \frac{z_j-p}{p} - \prod_{j \in S} \frac{p - z_j}{1-p} \Big) \right)^2 \right] \\
&= \frac{1}{n^2} \E \left[
\sum_{i,i'=1}^n Y_i Y_{i'} \sum_{\substack{S \subseteq \M_i \\ |S| \leq \beta}} \Big( \prod_{j \in S} \frac{z_j-p}{p} - \prod_{j \in S} \frac{p - z_j}{1-p} \Big) \sum_{\substack{S' \subseteq \M_{i'} \\ |S'| \leq \beta}} \Big( \prod_{j' \in S'} \frac{z_{j'}-p}{p} - \prod_{j' \in S} \frac{p - z_{j'}}{1-p} \Big)
\right] \\
&=\frac{1}{n^2} \E \left[
\sum_{i,i'=1}^n Y_i Y_{i'} \sum_{\substack{S \subseteq \M_i \\ |S| \leq \beta}} \Big( \Psi(S) - \Psi'(S) \Big) \sum_{\substack{S' \subseteq \M_{i'} \\ |S'| \leq \beta}} \Big( \Psi(S') - \Psi'(S')  \Big)
\right]
\intertext{Putting in the value of $Y_i$, we get}
&=\frac{1}{n^2} \E \left[
\sum_{i,i'=1}^n \sum_{\substack{T \subseteq \N_i \\ |T| \leq \beta}} c_{i,T} Z(T)\sum_{\substack{T' \subseteq \N_{i'} \\ |T'|  \leq \beta}} c_{i,T'} Z(T') \sum_{\substack{S \subseteq \M_i \\ |S| \leq \beta}} \Big( \Psi(S) - \Psi'(S) \Big) \sum_{\substack{S' \subseteq \M_{i'} \\ |S'| \leq \beta}} \Big( \Psi(S') - \Psi'(S')  \Big)
\right]\\
&=\frac{1}{n^2} 
\sum_{i,i'=1}^n \sum_{\substack{T \subseteq \N_i \\ |T| \leq \beta}} c_{i,T} \sum_{\substack{T' \subseteq \N_{i'} \\ |T'|  \leq \beta}} c_{i,T'} \E \left[ Z(T) Z(T') \sum_{\substack{S \subseteq \M_i \\ |S| \leq \beta}} \Big( \Psi(S) - \Psi'(S) \Big) \sum_{\substack{S' \subseteq \M_{i'} \\ |S'| \leq \beta}} \Big( \Psi(S') - \Psi'(S')  \Big)
\right]\\
&=\frac{1}{n^2} 
\sum_{i,i'=1}^n \sum_{\substack{T \subseteq \N_i \\ |T| \leq \beta}} c_{i,T} \sum_{\substack{T' \subseteq \N_{i'} \\ |T'|  \leq \beta}} c_{i,T'} \sum_{\substack{S \subseteq \M_i \\ |S| \leq \beta}} \sum_{\substack{S' \subseteq \M_{i'} \\ |S'| \leq \beta}}  \E \left[ Z(T) Z(T') \Big( \Psi(S) - \Psi'(S) \Big) \Big( \Psi(S') - \Psi'(S')  \Big)
\right] \\
&=\frac{1}{n^2} 
\sum_{i,i'=1}^n \sum_{\substack{T \subseteq \N_i \\ |T| \leq \beta}} c_{i,T} \sum_{\substack{T' \subseteq \N_{i'} \\ |T'|  \leq \beta}} c_{i,T'} \sum_{\substack{S \subseteq \M_i \\ |S| \leq \beta}} \sum_{\substack{S' \subseteq \M_{i'} \\ |S'| \leq \beta}}  \E \left[ Z(T) Z(T') \Big( \Psi(S) \Psi(S') + \Psi'(S) \Psi'(S') + \right. \\
&\left. \qquad  \qquad \qquad \qquad  \qquad \qquad \qquad  \qquad \qquad \qquad  \qquad \qquad  - \Psi(S) \Psi'(S') - \Psi'(S) \Psi'(S)  \Big) 
\right]
\end{align*}
By applying Lemma , we consider in the expectation only the terms with sets $S,S',T,T'$ such that $S\Delta S' \: \subseteq \:\: T \cup T'$. Moreover when these values are non-zero, they are upper bounded by $(p(1-p))^{-\beta}$

Additionally we have the following
\begin{align*}
\E[Z(T)Z(T')\Psi(S)\Psi(S')] &= p^{|T'' \setminus (S \cup S')|} (\frac{(1-p)^2}{p})^{|T'' \cap S \cap S'|} (\frac{1-p}{p})^{|(S'\cap S) \setminus T''|} \Ind[ S \Delta S' \subseteq T''] \\
&\leq (\frac{1-p}{p})^{|(S'\cap S) \setminus T''|} \Ind[ S \Delta S' \subseteq T'']     
\end{align*}

Similarly, 

$$\E[Z(T)Z(T')\Psi'(S)\Psi'(S')] \leq (\frac{(p)}{1-p})^{|(S'\cap S) \setminus T''|} \Ind[ S \Delta S' \subseteq T''] $$

$$\E[Z(T)Z(T')\Psi(S)\Psi'(S')] = p^{|T'' \setminus (S \cup S')} (1-p)^{|T'' \cap S \cap S'|} (-1)^{|(S'\cap S) \setminus T''|} \Ind[ S \Delta S' \subseteq T''] \in (-1,1)$$
where $T'' = T \cup T'$

Since the first two terms $\E[Z(T)Z(T')\Psi(S)\Psi(S')] \; \E[Z(T)Z(T')\Psi'(S)\Psi'(S')]$ add together in the variance, this leads to a bound of $\left( (\frac{1-p}{p})^{|(S'\cap S) \setminus T''|} +  (\frac{(p)}{1-p})^{|(S'\cap S) \setminus T''|} \right)  \Ind[ S \Delta S' \subseteq T'']$.

Next note that none of the sets $T,T',S,S'$ can be bigger than $\beta$ elements. Under this constraint the bound in the previous equation is maximized by setting $S=S' \text{ and } S\cap S' \text{ disjoint from } T''$. Putting the corresponding sizes of the sets, we get a bound of 
$$ \left( (\frac{1-p}{p})^{\beta} +  (\frac{(p)}{1-p})^{|\beta|} \right) \leq (p(1-p))^{-\beta} $$

Putting these in the equation for variance we get

$$
\Var[\hat{\tau}^\beta]  \leq \frac{1}{n^2} \sum_{i,i'} 
\sum_{\substack{T \subseteq \N_i \\ |T| \leq \beta}} c_{i,T} \sum_{\substack{T' \subseteq \N_{i'} \\ |T'|  \leq \beta}} c_{i,T'} \sum_{\substack{S \subseteq \M_i \\ |S| \leq \beta}} \sum_{\substack{S' \subseteq \M_{i'} \\ |S'| \leq \beta}}(p(1-p))^{-\beta}\Ind[ S \Delta S' \subseteq T \cup T']$$

If  $i' \notin \M^2_i$ then the sets $S,S'$ in the  expectation have no overlap, and hence can be ignored. Furthermore $|\M^2_i| \leq d^2_\M$, where $d_\M$ is the max degree of a node. Similarly, the coefficients $c$ can only occur from $\N^2$. We also see that the number of $T,T'$ is bounded by $d^\beta_\N$. Finally for any cluster of sets $S,S',T,T'$ that contribute, they can be enumerated by choosing a set corresponding to $T\cup T'$ of size $k$, and seperately counting the elements in $S,S'$. Hence this sum is upper bounded by
 
\begin{align}
\Var[\hat{\tau}^\beta]& \leq \frac{1}{n^2} \sum_i Y^2_{\max} \sum_{i' \in \N^2_i} \sum_{\substack{S \subseteq \M_i \\ |S| \leq \beta}} \sum_{\substack{S' \subseteq \M_{i'} \\ |S'| \leq \beta}} (p(1-p))^{-\beta}
\leq \frac{1}{n} Y^2_{\max} (p(1-p))^{-\beta} d_\M^{\beta+1}  d_\N^{\beta+1}
\label{apx:eqn:var_bound}
\end{align}

Next, we add back the additional uncertainty in $Y_i$ due to $\epsilon$. If $\epsilon$ is an entirely independent error with variance $\sigma^2$, we get an additional term of the form $\E \left[ \epsilon^2 \Big( \Psi(S) \Psi(S') + \Psi'(S) \Psi'(S') - \Psi(S) \Psi'(S') - \Psi'(S) \Psi'(S)  \Big) \right] $.
By replace $c_{i,T/T'}$ by $\epsilon$, and working through the same derivation we have this additional term being bounded above by $\frac{\sigma^2}{n}d_\M^{2\beta}  d_\N^{-\beta}(p(1-p))^{-\beta}$

If $Y_{\max}$ and $d_\M$ are bounded, then the above variance asymptotes to 0 as $n \rightarrow \infty$. Thus as  $n \rightarrow \infty$, the estimator $\hat{\tau}^\beta$ is a.s. constant. Thus the previous theorem on variance bound combined with the earlier proves unbiasedness shows the asymptotic consistency of the estimators $\hat{\tau}_{Lin}$ and  $\hat{\tau}^\beta$. (i.e. Theorems \ref{thm:beta} and \ref{thm:lin})

\end{proof}


\begin{thm} (same as Theorem \ref{thm:bias} of main text)
For the non-linear model \ref{eqn:nonlinear}, if $g_i$ is $k+1$ times differentiable and the $k+1$th derivative is bounded by $C$, then the absolute bias  $$\left|\E[\hat\tau^\beta] - \tau(\vec{1},\vec{0})\right| \leq \dfrac{C\max(p,1-p)}{(k+1)!}.$$
\end{thm}

\begin{proof}
    By Taylor's theorem $Y_i(\bz)$ can be written as
    a polynomial of order $k$ in $\bz[\N_i]$ with a $k+1$ order residual term
    $$ Y_i(\bz) = g_i^0 +  \sum_{z_j} \partial_{z^j} g_i z_j  +  \sum_{z_j,z_k} \partial^2_{z_j,z_k} g_i z_jz_k + .. + R_{k+1} = Y_{i,k}(\bz) + R_{k+1}$$
    where $R_{k+1} = \frac{\partial^{k+1}}{k+1!}Y_i(z^*) poly(\bz[\N_i],k+1)$ for some $z^* \in [0,1]^{|\N_i|}$

For simplicity we denote $\Big( \prod_{j \in S} \frac{z_j-p}{p} - \prod_{j \in S} \frac{p - z_j}{1-p} \Big)$ as $\Psi_{S})$
Now we use the result S2 from the previous theorem with $S'$ being a set of size $k+1$ and $\beta=k$ we get 
\begin{align*}
\E[ z_1z_2..z_{k+1}
\sum_{\substack{S \subseteq \M_i \\ |S| \leq k}} \Psi_{S}] &=  p^{k+1} \sum_{\substack{r \leq k}}  {k+1 \choose r}  \left[ (\frac{1}{p} - 1)^{r}  -  (-1)^{r}
\right] \\
&= p^{k+1}\left( \frac{1}{p}^{k+1} - (\frac{1}{p} - 1)^{k+1} + (-1)^{k+1} \right)
\end{align*}

Since $Y_{i,k}(\bz)$ is estimated without bias by our method, the bias of using the $k$ order approximated is given by
Hence 

\begin{align*}
| \E[R_{k+1} \sum_{\substack{S \subseteq \M_i \\ |S| \leq k}} \Psi_{S} )] | &=| \E[\frac{\partial^{k+1}}{k+1!}Y_i(z^*) poly(\bz[\N_i],k+1) \Psi_{S} ] | \\
&\leq |\frac{C}{k+1!}| \E[poly(\bz[\N_i],k+1) \Psi_{S}] |  \\
&\leq  |\frac{C}{k+1!}| p^{k+1}\left( \frac{1}{p}^{k+1} - (\frac{1}{p} - 1)^{k+1} + (-1)^{k+1} \right)|\\
&= O|\frac{Cp}{k+1!}|
\end{align*}

Similarly $Y_i(\bz)$ can be expanded around $\vec{1}$ and a similar derivation follows with $z_i$ replaced by $1-z_i$, and $p$ replaced by $1-p$. Putting both together derives the bound in the theorem
\end{proof}

\section{Proof of Self-Normalized and Doubly Robust Estimators}


Since the linear estimator is just a subcase of the general estimate $\hat{\tau}^\beta_{WIS}$, we focus on the more general version.

\begin{thm}
Under assumptions $\textbf{A2-5}$ and $\M_i \supseteq \N_i$, then $\hat{\tau}^\beta_{WIS}$ is asymptotically consistent
\end{thm}

\begin{proof}

Let $\rho_i = z_j/p$ and $\rho'_i = \frac{1-z_j}{1-p}$. Moreover let 
$\bar{\rho}_j$ and $\bar{\rho}'_j$ be the corresponding self normalized values of $\rho$ and $\rho'$, as in \ref{eq:self_normal_rho}. The self-normalized estimate for $\hat{\tau}^\beta$ is given by:
\vspace*{-0.25cm}
\begin{align*}
\hat{\tau}^{\beta}_{\text{WIS}} = \frac{1}{n} \sum_{i=1}^n Y_i \sum_{\substack{S \subseteq \M_i \\ |S| \leq \beta}} \Big( \prod_{j \in S} \big(\bar{\rho}_j -1 \big) - \prod_{j \in S} \big(-1 + \bar{\rho}'_j \big) \Big)
\end{align*}
\vspace*{-0.6cm}

Note that as all $z_i$ are independent, so are all $\rho,\rho'$. As such we can apply Kolmogorove's strong law of large numbers to any linear combination of these variables and get the following
\vspace*{-0.25cm}
\begin{equation*}
\begin{aligned}
\lim_{n \rightarrow \infty} \frac{1}{n}\sum\limits_j \sum\limits_{k\in M_j} \frac{1}{\M'_k} \rho_k &\overset{a.s.}{\longrightarrow}
\sum\limits_j \sum\limits_{k\in M_j} \frac{1}{\M'_k} \E[\rho_k] = 1
\end{aligned}
\qquad\quad
\begin{aligned}
\lim_{n \rightarrow \infty} \frac{1}{n}\sum\limits_j \sum\limits_{k\in M_j} \frac{1}{\M'_k} \rho'_k &\overset{a.s.}{\longrightarrow}
\sum\limits_j \sum\limits_{k\in M_j} \frac{1}{\M'_k} \E[\rho'_k] = 1
\end{aligned}
\end{equation*}
\vspace*{-0.3cm}

Since $1/t$ is a continuous function at $t=1$, we can apply Slutsky continuous mapping theorem on \ref{eq:self_normal_rho} to get that
\vspace*{-0.3cm}
$$ \text{as} \lim_{n \rightarrow \infty}: \:\:\: \bar{\rho}_i \rightarrow \rho_i \text{ and  } \bar{\rho'}_i \rightarrow \rho'_i.$$
This proves the consistency of $\hat{\tau}^\beta_{WIS}$ (Theorems \ref{thm:wis} and \ref{thm:betawis} of main text). Moreover as $\bar{\rho} \neq \rho$, the expected value of $\hat{\tau}^\beta_{WIS}$ will not generally equal $\hat{\tau}^\beta$.
\end{proof}

Next we demonstrate the unbiasedness and consistency of the DR estimator. Once again we focus on the general $\beta$ case, as the linear case is a corollary.
\begin{thm}
Under assumptions $\textbf{A2-5}$, and assuming $\mathcal{M}_i \supseteq \mathcal{N}_i$ , $\E[\hat{\tau}^\beta_{DR}] = \tau(\vec{1},\vec{0})$ and $\hat{\tau}_{DR}^\beta \overset{a.s.}{\longrightarrow}\tau(\vec{1},\vec{0})$. 
\end{thm}
\begin{proof}
Since $X \indep Z \; \E[f(X)g(Z)] =  \E[f(X)]\E[g(Z)]$. Applying this on
Equation \ref{eq:beta_dr_def}, we get
\begin{align*}
    \E[\hat{\tau}^\beta_{DR}] &= \E \left [ \frac{1}{n} \sum_{i=1}^n \tilde{Y}_i  \sum_{\substack{S \subseteq \M_i \\ |S| \leq \beta}} \Big( \prod_{j \in S} \big(\rho_j -1 \big) - \prod_{j \in S} \big(-1 + \rho'_j \big) \Big) \right] + \frac{1}{n} \sum_{i} \E [f_1(X_i) - f_0(X_i)] \\
    &= \E \left [ \frac{1}{n} \sum_{i=1}^n (Y_i - z_i(f_1(X_i)) - (1-z_i)f_0(X_i))  \sum_{\substack{S \subseteq \M_i \\ |S| \leq \beta}} \Big( \prod_{j \in S} \big(\rho_j -1 \big) - \prod_{j \in S} \big(-1 + \rho'_j \big) \Big) \right] \\
    & \qquad \qquad \qquad + \frac{1}{n} \sum_{i} \E [f_1(X_i) - f_0(X_i)]\\
    &\overset{a}{=} \E \left [ \frac{1}{n} \sum_{i=1}^n Y_i \sum_{\substack{S \subseteq \M_i \\ |S| \leq \beta}} \Big( \Psi(S) - \Psi'(S) \Big) \right] - 
    \E[(f_1(X))] \E[ z_i \sum_{\substack{S \subseteq \M_i \\ |S| \leq \beta}} \Big( \Psi(S) - \Psi'(S) \Big)] \\
    &\qquad\qquad\qquad - \E[(f_0(X))] \E[ (1- z_i)\sum_{\substack{S \subseteq \M_i \\ |S| \leq \beta}} \Big( \Psi(S) - \Psi'(S) \Big)]  + \frac{1}{n} \sum_{i} \E [f_1(X_i) - f_0(X_i)]\\
    &= \tau(\vec{1},\vec{0]}  - \E[f_0(X)] (1) - \E[f_1(X)](-1) + \frac{1}{n} \sum_{i} \E[f_1(X_i) - f_0(X_i)] = \tau(\vec{1},\vec{0})
\end{align*}
Effectively, $\hat{\tau}_{DR}$, replaces $c_{i,\{i\}}$ by $c_{i,\{i\}} - f_{z_i}(X)$. 
As long as these $f_{0/1}(X)$ are bounded, the same analysis of variance holds with the modified $c$. Thus $\hat{\tau}^\beta_{DR}$ inherits consistency from $\hat{\tau}^\beta$
\end{proof}

\begin{proposition}
If the outcome models $f(X)$ are accurate, the DR estimator $\hat{\tau}_{DR}^\beta$ have lower variance than $\hat{\tau}^\beta$.
\end{proposition}
This can be seen from the bound for variance in   Eqaution \ref{apx:eqn:var_bound}. Note that the bound is in terms of the variance of the independent noise $\sigma$ and the magnitude of the  outcome $\max |Y|$. The DR estimates center the outcomes at the outcome models, $f(X)$ , the bound for the DR estimator is determined by $\max |Y - f(X)|$, which if the outcome models are good should be smaller than $\max |Y|$.

This term provides support for the idea that with a good model so that $Y_i - f(X_i)$ has smaller coefficients $c_i$ (and hence lower $Y_{\max}$) or a lower error in $\epsilon$, the variance of the DR estimator can be lower than the standard estimate.


\section{Relationship with HT Estimator}
\label{apx:HTconnection}
\begin{align}
    \hat{\tau}_{\text{HT}} &= \frac{1}{n} \sum_i Y_i   \left(  \prod_{j \in \mathcal{N}_i} \frac{z_j}{p}  - \prod_{j \in \mathcal{N}_i} \frac{(1-z_j)}{(1-p)}  \right)
    \\
    \E [\hat{\tau}_{\text{HT}}] &= \frac{1}{n}\sum_i \E \left[  Y_i     \left(  \prod_{j \in \mathcal{N}_i} \frac{z_j}{p}  - \prod_{j \in \mathcal{N}_i} \frac{(1-z_j)}{(1-p)}  \right) \right] 
    \\
    &= \frac{1}{n}\sum_i \E \left[   \left(\sum_{k\in \mathcal{N}_i} c_{ik} \mathbb I[z_k=1]\right)     \left(  \prod_{j \in \mathcal{N}_i} \frac{z_j}{p}  - \prod_{j \in \mathcal{N}_i} \frac{(1-z_j)}{(1-p)}  \right)\right] 
    \\
    &= \frac{1}{n}\sum_i \sum_{k\in \mathcal{N}_i}  \E \left[   c_{ik} \mathbb I[z_k=1]  \left(  \prod_{j \in \mathcal{N}_i} \frac{z_j}{p}  - \prod_{j \in \mathcal{N}_i} \frac{(1-z_j)}{(1-p)}  \right) \right]  
    \\
    &= \frac{1}{n}\sum_i \sum_{k\in \mathcal{N}_i}  \E \left[   c_{ik}  \mathbb I[z_k=1]   \frac{z_k}{p} \right]   \underbrace{\E \left[    \prod_{j \in \mathcal{N}_i, j\neq k} \frac{z_j}{p} \right]}_{=1}  
    -  \frac{1}{n}\sum_i \sum_{k\in \mathcal{N}_i}  \E \left[   c_{ik}  \mathbb I[z_k=1]   \frac{1- z_k}{1-p} \right]   \underbrace{\E \left[    \prod_{j \in \mathcal{N}_i, j\neq k} \frac{(1-z_j)}{(1-p)}  \right]}_{=1}  
    \\
    &= \frac{1}{n}\sum_i \sum_{k\in \mathcal{N}_i}  c_{ik}  \E \left[   \mathbb I[z_k=1]  \left( \frac{z_k}{p} -  \frac{(1-z_k)}{(1-p)}\right) \right]
    \\
    &=  \E \left[  \frac{1}{n}\sum_i \sum_{k\in \mathcal{N}_i}  c_{ik}   \mathbb I[z_k=1]  \left( \frac{z_k}{p} -  \frac{(1-z_k)}{(1-p)}\right) \right]
    \\
    &\overset{(a)}{=}  \E \left[  \frac{1}{n}\sum_i \left(\sum_{k\in \mathcal{N}_i}  c_{ik}   \mathbb I[z_k=1] \right)
    \left( \sum_{k\in \mathcal{N}_i} \frac{z_k}{p} -  \frac{(1-z_k)}{(1-p)}\right) \right]
    \\
    &= \E \left[ \underbrace{ \frac{1}{n}\sum_i Y_i
    \left( \sum_{k\in \mathcal{N}_i} \frac{z_k}{p} -  \frac{(1-z_k)}{(1-p)}\right)}_{\hat{\tau}_{\text{Lin}}} \right] 
\end{align}
where (a) follows from observing that the sum of cross product terms is 0,
\begin{align}
    \E \left[ \sum_{k\in \mathcal{N}_i}  c_{ik}   \mathbb I[z_k=1] 
\sum_{j\in \mathcal{N}_i, j\neq k} \frac{z_j}{p} -  \frac{(1-z_j)}{(1-p)}) \right] = \E \left[ \sum_{k\in \mathcal{N}_i}  c_{ik}   \mathbb I[z_k=1] 
\underbrace{\E \left[\sum_{j\in \mathcal{N}_i, j\neq k} \frac{z_j}{p} -  \frac{(1-z_j)}{(1-p)} \right]}_{=0} \right] = 0.
\end{align}

\pagebreak

\section{Additional Experiment Results}

\begin{figure*}[th!]
    \centering
    \begin{subfigure}[b]{0.25\textwidth}
    \centering
    \includegraphics[width=\textwidth]{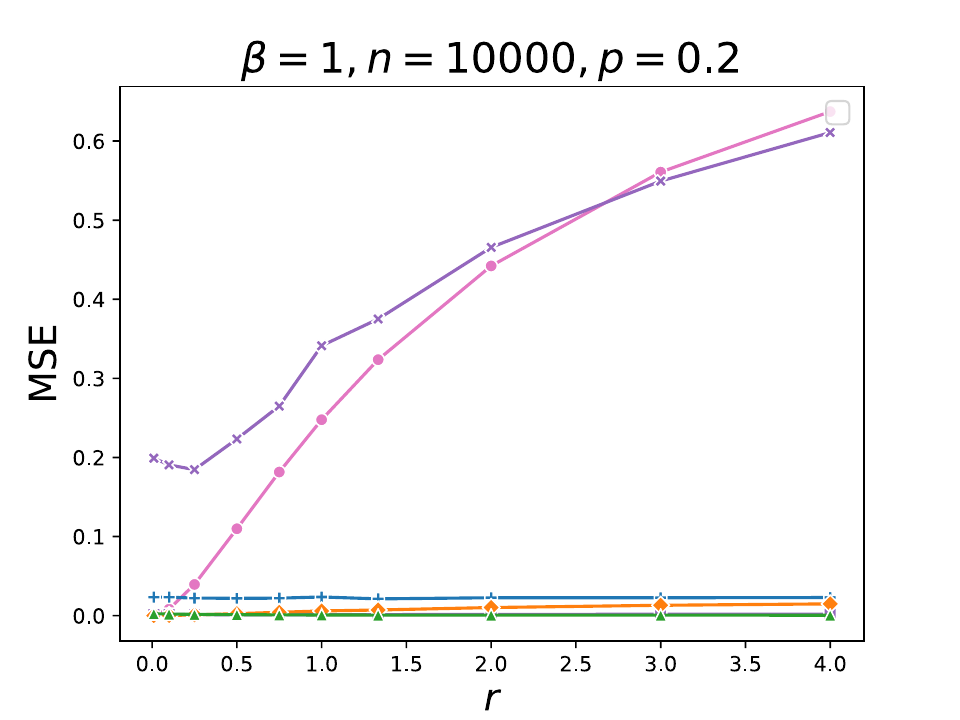}
    \includegraphics[width=\textwidth]{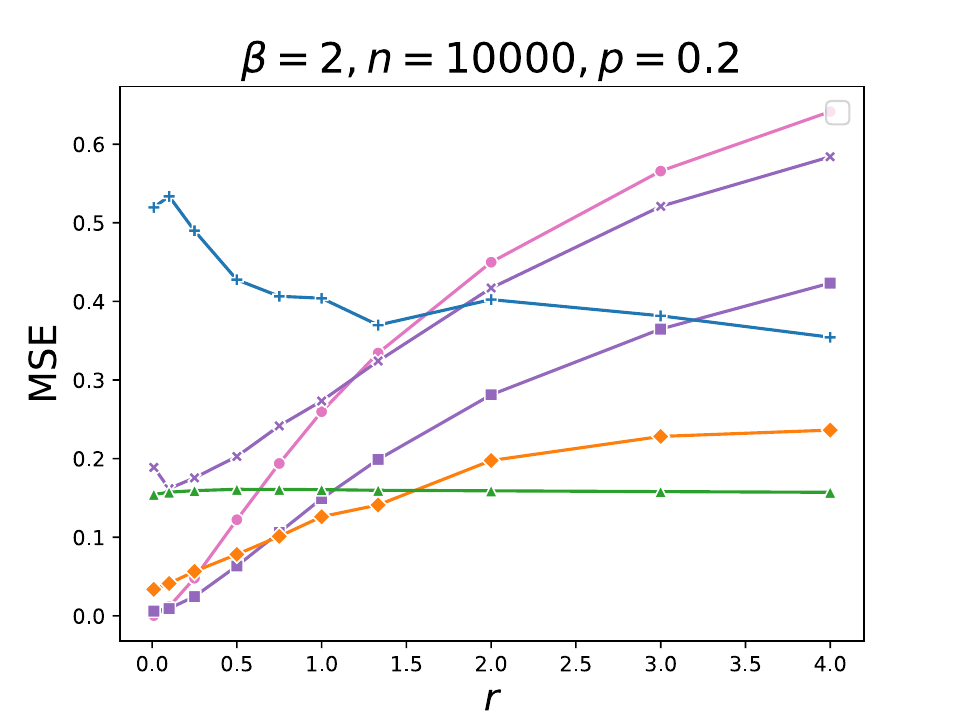}
    \caption{Direct/indirect effects: r }  \label{fig:mseratioER}
    \end{subfigure}
    ~
    \begin{subfigure}[b]{0.25\textwidth}
    \centering
    \includegraphics[width=\textwidth]{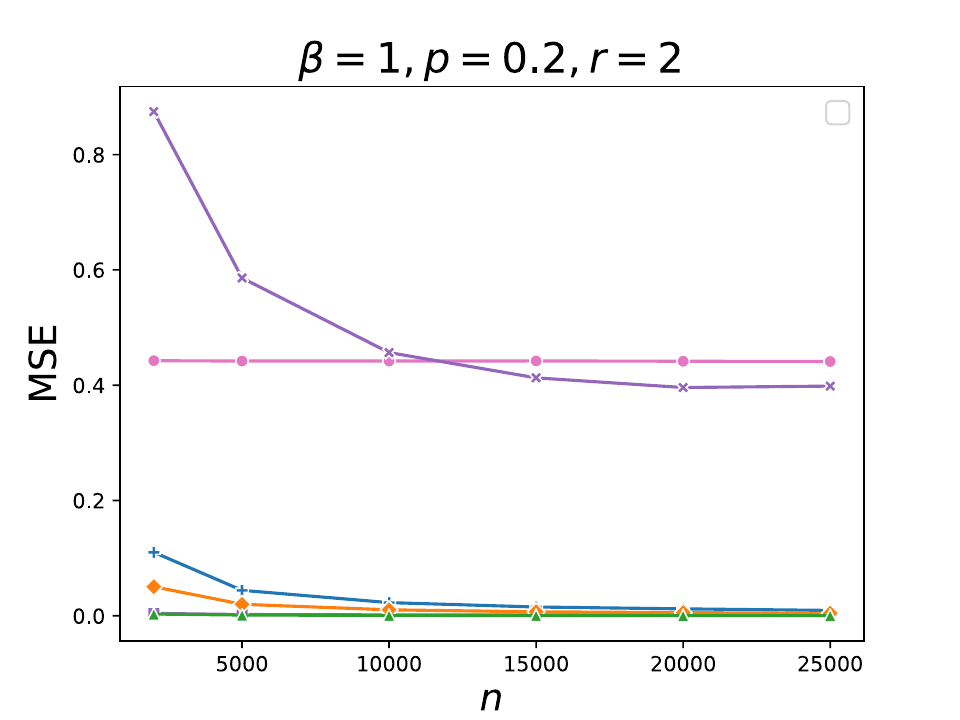}
    \includegraphics[width=\textwidth]{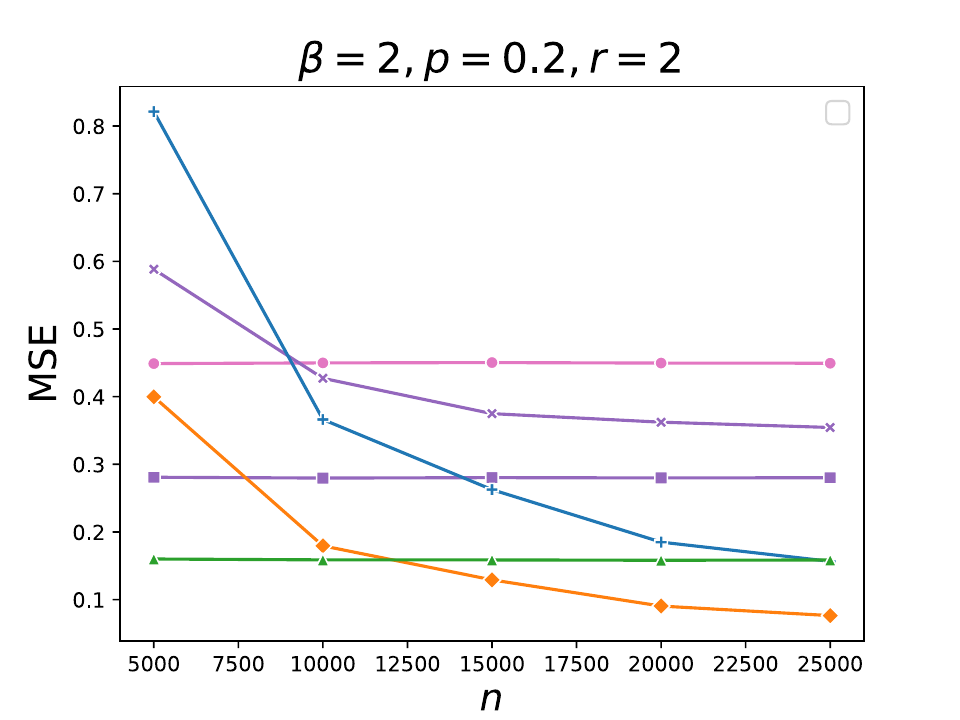}
    \caption{Population size: n}  \label{fig:msesizeER}
    \end{subfigure}
    ~
    \begin{subfigure}[b]{0.25\textwidth}
    \centering
    \includegraphics[width=\textwidth]{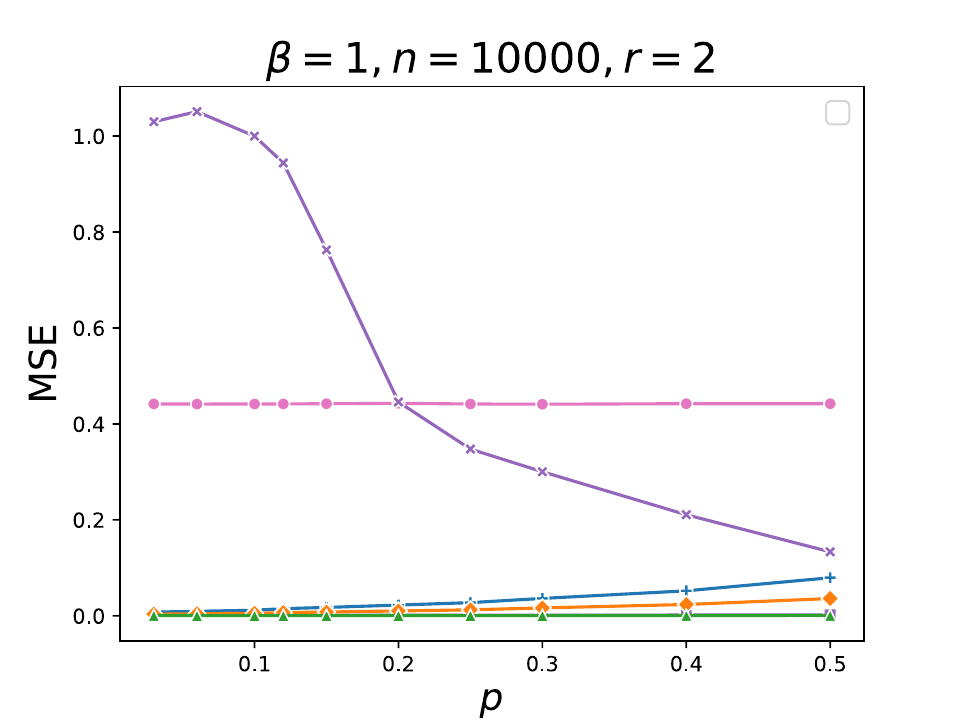}
    \includegraphics[width=\textwidth]{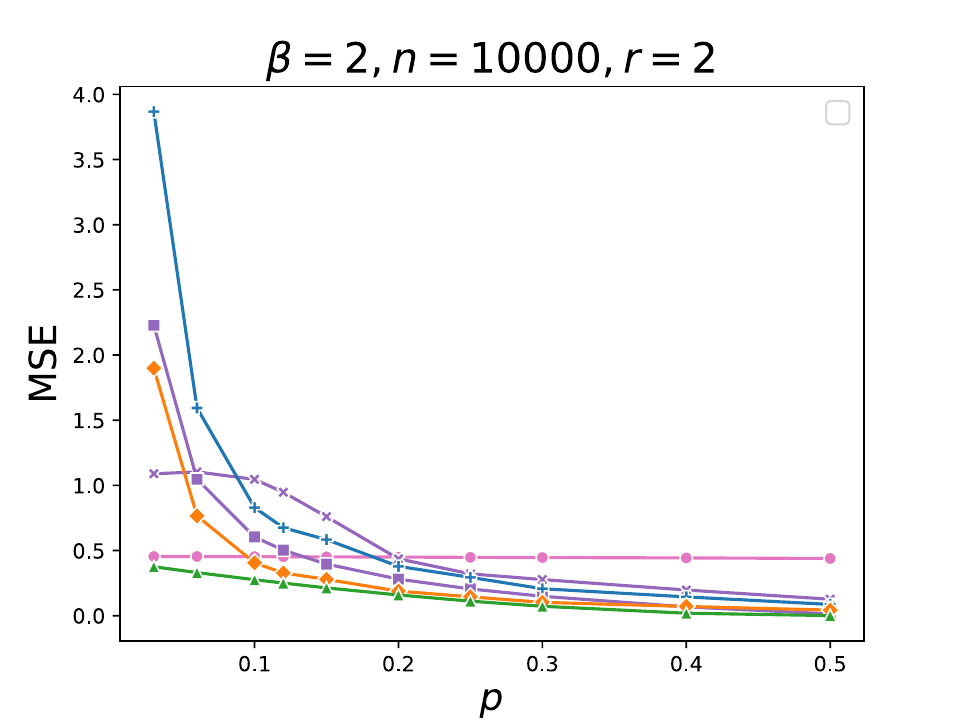}
     \caption{Treatment budget: p}  \label{fig:msepER}
    \end{subfigure}
    \begin{subfigure}[b]{0.9\textwidth}
    \includegraphics[width=\textwidth]{figs4/legend1.pdf}
    \end{subfigure}
    \vspace{-20pt}
    \caption{Performance of GATE estimators under Bernoulli design on
Erdos-Renyi networks. Y-axes represent the MSE.  The rows correspond to different generative model for the potential outcomes.
Columns correspond to different parameters being varied: (a) Strength of interference,  where the x-axis corresponds to the average ratio of indirect to direct effects $r = \frac{1}{n}\sum_{i,j}\frac{|c_{ij}|}{{c_{ii}}|}$. (b) Population size,  where the x-axis corresponds to the number of nodes $n$ (c) Treatment budget, where the x-axis corresponds to the probability of treatment 1 ($p$).
\label{fig:erdos_mse_combined}}
\end{figure*}
\begin{figure}[th!]
    \centering
    \begin{subfigure}[b]{0.23\textwidth}
    \centering
    \includegraphics[width=\textwidth]{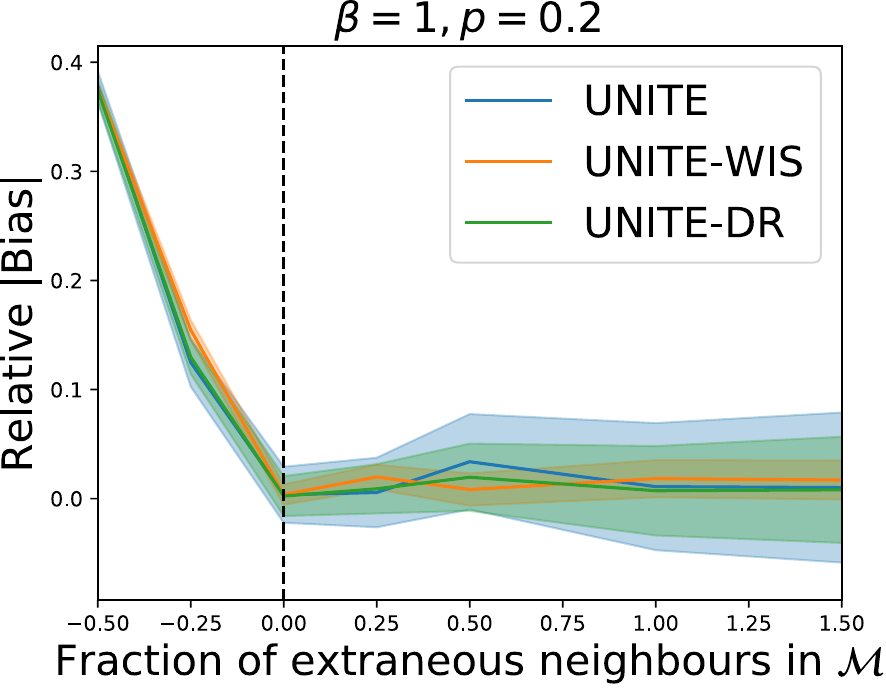}
    \caption{ Bias graphs\label{fig:abla_ernet2}}
    \end{subfigure}
    ~
    \begin{subfigure}[b]{0.23\textwidth}
     \centering
     \includegraphics[width=\textwidth]{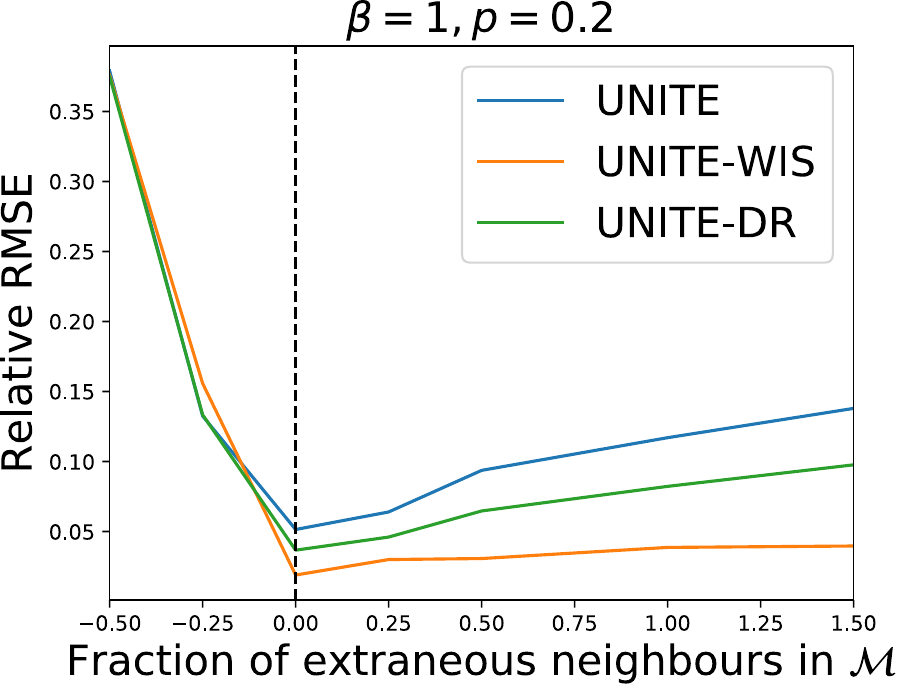}
     \caption{RMSE \label{fig:abla_airbnb}}
     \end{subfigure}
    \caption{Visualization of the impact of neighbourhood sizes on GATE estimation. Negative fraction of neighbours indicate the case when $\mathcal{M}(i) \subset \mathcal{N}_i$ i.e. we missed pertinent neighbours. The bias tends to be high when gives small neighbourhoods, as they miss pertinent edges. As the neighbourhood sizes increase, the bias reduces, but the uncertainty widens.
    \label{fig:ablationa} 
    }
\end{figure}

\section{Checklist}



 \begin{enumerate}

 \item For all models and algorithms presented, check if you include:
 \begin{enumerate}
   \item A clear description of the mathematical setting, assumptions, algorithm, and/or model. [Yes]
   \item An analysis of the properties and complexity (time, space, sample size) of any algorithm. [Yes]
   \item (Optional) Anonymized source code, with specification of all dependencies, including external libraries. [Yes]
 \end{enumerate}

 \item For any theoretical claim, check if you include:
 \begin{enumerate}
   \item Statements of the full set of assumptions of all theoretical results. [Yes]
   \item Complete proofs of all theoretical results. [Yes]
   \item Clear explanations of any assumptions. [Yes]     
 \end{enumerate}

 \item For all figures and tables that present empirical results, check if you include:
 \begin{enumerate}
   \item The code, data, and instructions needed to reproduce the main experimental results (either in the supplemental material or as a URL). [Yes]
   \item All the training details (e.g., data splits, hyperparameters, how they were chosen). [Yes]
         \item A clear definition of the specific measure or statistics and error bars (e.g., with respect to the random seed after running experiments multiple times). [Yes]
         \item A description of the computing infrastructure used. (e.g., type of GPUs, internal cluster, or cloud provider). [No]
 \end{enumerate}

 \item If you are using existing assets (e.g., code, data, models) or curating/releasing new assets, check if you include:
 \begin{enumerate}
   \item Citations of the creator If your work uses existing assets. [Yes]
   \item The license information of the assets, if applicable. [Not Applicable]
   \item New assets either in the supplemental material or as a URL, if applicable. [Not Applicable]
   \item Information about consent from data providers/curators. [Not Applicable]
   \item Discussion of sensible content if applicable, e.g., personally identifiable information or offensive content. [Not Applicable]
 \end{enumerate}

 \item If you used crowdsourcing or conducted research with human subjects, check if you include:
 \begin{enumerate}
   \item The full text of instructions given to participants and screenshots. [Not Applicable]
   \item Descriptions of potential participant risks, with links to Institutional Review Board (IRB) approvals if applicable. [Not Applicable]
   \item The estimated hourly wage paid to participants and the total amount spent on participant compensation. [Not Applicable]
 \end{enumerate}

 \end{enumerate}

\end{document}